%% file: main.tex
\documentclass[wcp]{jmlr_arxiv}

\usepackage{hyperref}       
\usepackage{booktabs}       
\usepackage{amsfonts}       
\usepackage{mathtools}
\usepackage{nicefrac}       
\usepackage{algorithm}
\usepackage{graphicx}
\usepackage{,booktabs,array}
\usepackage{longtable}
\usepackage{mwe}
\usepackage[export]{adjustbox}

\def\ouralgo{{ProtoBandit}}
\newcolumntype{C}{>{\centering\arraybackslash}m{0.5\textwidth}}

\newcommand\defeq{\stackrel{\mathclap{\normalfont\mbox{\tiny{def}}}}{=}}

\DeclareMathOperator*{\argmax}{arg\,max}
\DeclareMathOperator*{\argmin}{arg\,min}
\DeclarePairedDelimiter{\ceil}{\lceil}{\rceil}
\DeclareMathOperator{\E}{\mathop{\mathbb{E}}}

\def\S{\mathcal{S}}
\def\T{\mathcal{T}}

\newcommand{\revision}[1]{{\color{red} #1}}

\makeatletter
\let\Ginclude@graphics\@org@Ginclude@graphics 
\makeatother

\jmlryear{2022}
\jmlrworkshop{Preprint}

\title[ProtoBandit: Efficient Prototype Selection via Multi-Armed Bandits]{ProtoBandit: Efficient Prototype Selection\\ via Multi-Armed Bandits}

  \author{\Name{Arghya {Roy Chaudhuri}} \Email{arghyar@microsoft.com}\\
  \Name{Pratik Jawanpuria} \Email{pratik.jawanpuria@microsoft.com}\\
  \Name{Bamdev Mishra} \Email{bamdevm@microsoft.com}\\
 \addr Microsoft India}

\jmlrvolume{0}

\newif\iflongversion
\longversiontrue 

\newcommand{\conditional}[2]{\iflongversion{#1}\else{#2}\fi}

\begin{document}

\maketitle

\input{abstract}

\input{manuscript}

\bibliography{references}
\clearpage
\appendix
\input{appendix}

\end{document}

%% file: abstract.tex
\begin{abstract}
In this work, we propose a multi-armed bandit-based framework for identifying a compact set of informative data instances (i.e., the prototypes) from a source dataset $\S$ that best represents a given target set $\T$. Prototypical examples of a given dataset offer interpretable insights into the underlying data distribution and assist in example-based reasoning, thereby influencing every sphere of human decision-making. 
Current state-of-the-art prototype selection approaches require $O(|\S||\T|)$ similarity comparisons between source and target data points, which becomes prohibitively expensive for large-scale settings. 
We propose to mitigate this limitation by employing stochastic greedy search in the space of prototypical examples and multi-armed bandits for reducing the number of similarity comparisons. Our randomized algorithm, {\ouralgo}, identifies a set of $k$ prototypes incurring $O(\revision{k^3}|\S|)$ similarity comparisons, which is independent of the size of the target set. 
An interesting outcome of our analysis is for the $k$-medoids clustering problem ($\T = \S$ setting) in which we show that our algorithm {\ouralgo} approximates the \texttt{BUILD} step solution of the partitioning around medoids (PAM) method in $O(\revision{k^3}|\S|)$ complexity. 
Empirically, we observe that {\ouralgo} reduces the number of similarity computation calls by several orders of magnitudes ($100-1000$ times) while obtaining solutions similar in quality to those from state-of-the-art approaches.

\end{abstract}

%% file: manuscript.tex
\revision{Note: The published version of this paper \citep{roychaudhuri22a} contains an erratum in Theorem \ref{thm:disq_coplexity}, i.e., {\ouralgo} scales $O(\revision{k^3}|\S|)$ instead of $O(\revision{k}|\S|)$. This is because the $\nu_0$ of {\ouralgo} at every iteration and the overall $\nu$ (across all the iterations) related as $\nu_0 = \frac{\nu}{k(1 - 1/ e - \epsilon)}$ and not $\nu_0 = \frac{\nu}{(1 - 1/ e - \epsilon)}$ (as mentioned in the published version). We had earlier missed the factor $k$ in the denominator. Consequently, an extra factor of $k^2$ comes in the numerator in Theorem \ref{thm:disq_coplexity}. To this end, we have highlighted in red all the changes in this paper.}

\section{Introduction}~\label{sec:intro}
Prototypical examples are data instances that together summarizes a given dataset or an underlying data distribution~\citep{weiser1982programmers,sproto,bien11a,koh17,yeh18a}. Such compact representation of a dataset is especially useful in this age of big-data, where the size of datasets goes beyond the capability of manual checking. Hence, prototypical samples help domain experts and data scientists by providing meaningful insights in complex domains~\citep{proto}. They also provide example-based reasoning, thereby improving the interpretability of data distributions~\citep{kim14a,Kim16,bib:Gurumoorthy+JM:2021}. 


A popular use-case of prototype selection is to select a test group that best represent the control group (or vice-versa). 
For product design~\citep{bib:Leahy:2013}, a designer needs to know the required features for a new product directly from the targeted users. Therefore, it is important to select a small subset of users that best represent the larger customer base. In the healthcare domain, prototype selection has been employed for building training datasets~\citep{bib:Suyal+S:2021}. 
Another interesting application of prototype selection is in explainable AI~\citep{bib:Linardatos+PK:2021}, where we need to analyze the cause behind the output of an AI model. 

Works such as~\citep{sproto2, sproto} have proposed finding prototypical elements in supervised setting where both features and label information of the data points are available. 
However, recent works~\citep{proto,bib:Gurumoorthy+JM:2021,Kim16} have focused on the more general unsupervised setting where only feature information is available. In this paper, we focus on unsupervised prototype selection. A key challenge here is the big-data setting, where one may require to perform similarity computations between huge number of data points pairs~\citep{bib:Zukhba:2010}. This makes the existing algorithms~\citep{Kim16,proto,bib:Gurumoorthy+JM:2021} is impractical.
Below, we concretize the prototype selection problem setup and discuss our contributions.



\paragraph{Prototype selection problem setup.}
Given two non-empty sets of points (source) $\mathcal{S}$ and (target) $\mathcal{T}$, a dissimilarity measure $d: \mathcal{T} \times \mathcal{S} \mapsto [0, 1]$, and a positive integer $k \leq |\mathcal{S}|$, our aim is to find a set $\mathcal{M} \subseteq \mathcal{S}$ that best represents the target set $\mathcal{T}$ such that $|\mathcal{M}|\leq k$. The set $\mathcal{M}$ is the set of prototypical elements. We call the tuple $(\mathcal{S}, \mathcal{T}, \mathbf{q}, d, k)$ a instance of the prototype selection problem, where $\mathbf{q} \in [0, 1]^{|\mathcal{T}|}$ is the weight vector associated with $T$. {Such a weight vector can be useful in applications where not every target sample is equally important and can also be used to capture the underlying distribution of the target set (if available). For the sake of convenience and interpretability, we assume $\sum_{b \in \mathcal{T}} q_b = 1$.} 

\paragraph{$k$-medoids clustering as a prototype selection problem.} Recently, \cite{bib:Gurumoorthy+JM:2021} have shown that optimal transport based prototype selection problem reduces to the $k$-medoids clustering problem when the target set is same as the source set ($\T=\S$). Unlike the more popular $k$-means clustering, the $k$-medoids clustering problem requires the cluster centers to be actual data points in the given dataset $\S$. Hence, $k$-medoids clustering is especially useful when interpretable cluster centers are desired. For instance, $k$-means centers of a group of images can visually be a random noise image~\citep{leskovec2020a,bib:Tiwari+ZMTPS:2020}. Another benefit of the k-medoids set up is that it only relies distances/similarities between data points. However, it should be noted that $k$-medoids clustering is a NP-hard problem in general~\cite{kmedoidsalgo}. Popular $k$-medoids algorithms include the partitioning around medoids (PAM) method \citep{bib:Kaufman+R:1990,rousseeuw09a} and approaches inspired by PAM such as FASTPAM1~\citep{kmedoidsalgo}, CLARA~\citep{bib:Kaufman+R:1990}, and CLARANS~\citep{ng02a}. Among them, FASTPAM1 guarantees the same solution as PAM and has a computational complexity of $O(|\S|^2)$. CLARA and CLARANS, on the other hand, output PAM-like solution. A recent randomized approach, BanditPAM~\citep{bib:Tiwari+ZMTPS:2020}, guarantees the same results as PAM with a (expected) computational complexity of $O(k|\S|\log(|\S|))$.


\paragraph{Our Contributions.}
We first generalize the \texttt{BUILD} step of the PAM  method~\citep{bib:Kaufman+R:1990} to allow for the scenarios when $\mathcal{S}\neq\mathcal{T}$. 
We next prove that the \texttt{BUILD} step of the proposed generalized PAM algorithm is equivalent to the existing optimal transport based SPOTgreedy~\citep{bib:Gurumoorthy+JM:2021} algorithm for prototype selection. This is interesting because the PAM algorithm can now also be understood from the lens of the optimal transport theory. Since the \texttt{BUILD} step of the proposed generalized PAM method is based on a greedy search procedure, we use it as a stepping stone to propose a sampling-based algorithm {\ouralgo} for the prototype selection problem. Specifically, we introduce the following novelties in the {\ouralgo} algorithm. 
\begin{enumerate}
    \item We employ random subset selection for more efficient greedy search as it reduces the search space (to find the next potential prototype) in the source set $\mathcal{S}$. 
    \item At each iteration, existing methods~\citep{Kim16,proto,bib:Gurumoorthy+JM:2021} require similarity computations for every element of the selected subset $\mathcal{M}$ with every element in the target set $\mathcal{T}$. Thus, if the target set $\mathcal{T}$ is large, or potentially infinite, then such approaches become impractical. We circumvent this issue by employing the multi-armed bandits (MAB) based sampling technique to estimate the similarity of a source point with the target set $\mathcal{T}$. 
\end{enumerate}
Our key technical result is that the similarity computations required by the proposed {\ouralgo} is independent of the target set size $|\mathcal{T}|$. Overall, {\ouralgo} requires $O(\revision{k^3}|S|)$ computations for obtaining $k$ prototypes from the source set $\S$ that represent the target set $\T$. 
We also provide an approximation guarantee of the prototypical set obtained by the proposed {\ouralgo} algorithm. In particular, let $f: 2^\mathcal{S} \mapsto [0, 1]$ denote the  similarity function between an input (candidate) set $\mathcal{M}$ and the target set $\mathcal{T}$ and $\mathcal{M}^*$ be the optimal solution. Then, we prove that  $f(\mathcal{M}_{ALG}) \geq \left(1 - e^{-1} - \epsilon \right)f(\mathcal{M}^*) - \nu$ with probability at least $1 -\delta$, where $\mathcal{M}_{ALG} \subset \mathcal{S}$ is output prototype set of our proposed {\ouralgo}, $\epsilon, \nu \in (0, 1)$ are the approximation parameters, and $\delta \in (0, 0.05)$ is the error threshold. 
A corollary to our main result is that {\ouralgo} approximates the \texttt{BUILD} step solution of the PAM method for the $k$-medoids clustering problem (i.e, the case when $\T=\S$) in $O(\revision{k^3}|\S|)$ complexity.

\section{Related work}\label{subsec:prev_work}
\paragraph{Prototype selection.}
{The problem of prototype selection has been mostly explored  \citep{sproto,crammer02a,wohlhart13a,wei15a} in the supervised learning setups where the label of the data points are available. 
Recent prototype selection approaches such as MMD-Critic~\citep{Kim16}, ProtoDash~\citep{proto}, and SPOTgreedy~\citep{bib:Gurumoorthy+JM:2021}, aim at obtaining the set of prototypical elements $\mathcal{M}\subset\mathcal{S}$ such that $\mathcal{M}$'s underlying distribution is close to that of the target $\mathcal{T}$. Such approaches are suitable in unsupervised settings as they assume availability of only the similarity (or distance) between pairs $(a,b)$, where $a\in\mathcal{S}$ and $b\in\mathcal{T}$. 
MMD-Critic and ProtoDash employ the MMD distance to capture the similarity between two distributions while SPOTgreedy uses the Wasserstein distance, i.e., the optimal transport framework. We describe the SPOTgreedy algorithm in detail in Appendix \conditional{\ref{subsec:spot_restate}}{B of our extended version \citep{chaudhuri22a}}. 
The above three approaches are based on greedy search and a common bottleneck in them is that, at every iteration of greedy search, they require $\mathcal{O}(|\mathcal{T}|)$ similarity computations, which becomes impractical as the target set $\mathcal{T}$ becomes large. To this end, we propose a multi-armed bandits based approach to alleviate this concern. 
}
\paragraph{PAM algorithm.} Prototype selection when the source and target sets are identical may specifically be viewed as identifying important data points in the given set, i.e., data summarisation. An intuitive way of unsupervised data summarisation is to apply a centroid-based clustering technique and choose the cluster centers as the representative of the whole dataset. A popular approach to accomplish this is $k$-medoids clustering using the partition around medoids (PAM) method~\citep{bib:Kaufman+R:1990}. PAM applies an exhaustive greedy search over the whole set through two main steps: \texttt{BUILD} and \texttt{SWAP}. During the \texttt{BUILD} step, PAM selects the medoids, and during the \texttt{SWAP} step it improves upon the already chosen medoids by replacing them with the new ones. 



\section{{\ouralgo}: An Efficient Prototype Selection Algorithm}~\label{subsec:gpam}\label{sec:rel_pam_spot}











In this section, we first generalize the PAM algorithm~\citep{bib:Kaufman+R:1990} to the prototype selection problem (i.e., the sets $\mathcal{S}$ and $\mathcal{T}$ are different). 
We then build on the generalized PAM algorithm and propose {\ouralgo}, which makes use of approximate greedy search and multi-armed bandit frameworks. 

\subsection{Generalized PAM}
We propose two modifications in the existing $k$-medoids clustering algorithm, PAM, to generalize it to the prototype selection setting. We detail the \texttt{BUILD} step of the generalized PAM algorithm in Algorithm~\ref{alg:genpam_build}. The full generalized PAM algorithm is in \conditional{Appendix~\ref{app:pam}}{Appendix~A of the extended version of this paper \citep{chaudhuri22a}}.


The first modification enables the algorithm to choose medoids from a set that is different from the set of points to be clustered. Therefore, unlike the original PAM algorithm, Algorithm~\ref{alg:genpam_build} can choose set of medoids from $\mathcal{S}$ to cluster the points in $\mathcal{T}$, even if $\mathcal{S} \neq \mathcal{T}$. The second modification lies in choosing the number of elements $r$ while updating the set of chosen points $\mathcal{M}$ during the \texttt{BUILD} step, which allows for additional flexibility and efficiency.
In our experiments, we consider $r=1$ setting unless specified otherwise. In this setting the medoids from the source set $\mathcal{S}$ are selected in a (strict) sequential manner, which has both theoretical and qualitative benefits.


\begin{algorithm}
\caption{\texttt{BUILD} step for the generalized PAM algorithm}\label{alg:genpam_build}
\DontPrintSemicolon
\SetKwInOut{Input}{Input}
\SetKwInOut{Output}{Output}
\Input{Problem instance given by $(\mathcal{S}, \mathcal{T}, \mathbf{q}, d, k)$, and a positive integers $r \in \{1, \cdots, k\}$.}
\Output{A set $\mathcal{M}$, such that $|\mathcal{M}| = k$}
\textbf{Initialization:} For each $j \in \mathcal{T}$, set $D_j = \infty$.\;

In this step we apply greedy strategy to choose an initial set of $k$ prototypes.\; 
$\mathcal{M} = \emptyset$\;

\While{$|\mathcal{M}| < k$}{
    Define gain vector $g$ with entries
    \begin{equation}\label{eq:gain_genpam_build}
        g_i = \sum_{j \in \mathcal{T}} q_j \max \{D_j - d(j, i), 0\},\;  \forall i \in \mathcal{S} \setminus \mathcal{M}.
    \end{equation}
    \uIf(\tcp*[h]{no meaningful prototype to add}){$\forall i \in \mathcal{S}\setminus \mathcal{M}, g_i = 0$}{
    $Q \defeq$ Choose $r$ elements from  $\mathcal{S}\setminus\mathcal{M}$ at random \; 
    }\Else{
    $Q \defeq$ Set of indices of top $r$ largest \emph{non-zero} elements in $g = \{g_i\}_{i=1}^{|\{\mathcal{S} \setminus \mathcal{M}\}|}$ \;
    }
    $\mathcal{M} = \mathcal{M} \cup Q$ and 
     $\forall j \in \mathcal{T}$ update $D_j$.\;
}
\end{algorithm}

Below, we show the equivalence between the \texttt{BUILD} step of the proposed generalized PAM algorithm (in Algorithm~\ref{alg:genpam_build}) and the SPOTgreedy algorithm~\citep{bib:Gurumoorthy+JM:2021}. To this end, we define $D_j$ to be the dissimilarity between a point $j$ and the closest object in $\mathcal{M}$, i.e., 
\begin{equation}\label{eq:mindist_medoid}
    D_j = \min_{i \in \mathcal{M}} \{d(j, i)\}\;\left[\text{wherein dissimilarity measure}\;d: \mathcal{T} \times \mathcal{S} \mapsto [0, 1]\right].
\end{equation}

\begin{lemma} \label{lem:spot_build_pam}
Given a problem instance $(\mathcal{S}, \mathcal{T}, \mathbf{q}, d, k)$, for a fixed value of the input parameter $r \in \{1, \cdots, k\}$, the $k$ points identified by {SPOT}greedy is identical with the set of selected point by Algorithm~\ref{alg:genpam_build}. 
\end{lemma}

The proof of Lemma~\ref{lem:spot_build_pam} is given in \conditional{Appendix~\ref{app:spot_build_pam}}{Appendix~B of the extended version of this paper \citep{chaudhuri22a}}. 
The objective of this lemma is to show that given a similarity matrix $Z \in [0, 1]^{m \times n}$ at each iteration Algorithm~\ref{alg:genpam_build} chooses a prototype by maximizing the following {function}:
\begin{equation}\label{eq:spot_eq10}
    f(\mathcal{M}) = \sum_{j \in \mathcal{T}} q_j \max_{i \in \mathcal{M}} Z_{ji}\;[\text{for } \mathcal{M} \subset \mathcal{S}],
\end{equation}
where $Z$ is a $|\mathcal{T}|\times |\mathcal{S}|$ similarity matrix between the target and the source sets, based on the given dissimilarity measure $d$. We compute similarity between $j$-th data point in $\mathcal{T}$ and $i$-th data point in $\mathcal{S}$ as 
\begin{equation}\label{eq:spot_sim}
    Z_{ji} = C - d(j, i), 
\end{equation}
where $C$ is a constant such that $C \geq \max_{i \in \mathcal{S}, j \in \mathcal{T}}\{d(j, i)\}$ and $C\geq1$. Note that  $f$ in (\ref{eq:spot_eq10}) is submodular \citep{bib:Gurumoorthy+JM:2021}.


For choosing $k$ prototypes, Algorithm~\ref{alg:genpam_build} incurs $O(|\mathcal{S}||\mathcal{T}|k/r)$ similarity comparisons, which is a computational bottleneck for large sets. We next propose an approximate greedy-search for Algorithm~\ref{alg:genpam_build} that scales independent of $|\mathcal{T}|$ and linearly in $|\mathcal{S}|$.




\subsection{Approximate Greedy Search}
\label{subsec:approx_greedy}

Due to simplicity and versatility of the greedy search framework, it has attracted a lot of attentions~\citep{bib:Minoux:1978,bib:Wei+IB:2014,bib:Badanidiyuru+V:2014,bib:Mirzasoleiman+BKVK:2015} for reducing the number of similarity comparisons. We choose Stochastic-Greedy~\citep{bib:Mirzasoleiman+BKVK:2015} as it is one of the simplest and empirically fast~\citep{bib:Mirzasoleiman+BKVK:2015} algorithms for approximate greedy search. Without approximation, at every step a greedy algorithm exhaustively searches through the set to find next the potential solution. Instead, {Stochastic-Greedy} considers a random subset for the same. The size of the subset is chosen depending on the approximation parameter $\epsilon$, such that the outcome is optimal within a $(1-e^{-1}-\epsilon)$ multiplicative factor.

{Stochastic-Greedy}~\citep{bib:Mirzasoleiman+BKVK:2015} assumes that the objective function evaluation can be done in $O(1)$. Also, for each point it examines, the exact value of the objective function has to be computed. Let us  understand its effect on the problem of prototype selection through an example. Assuming, $\mathcal{R} \subset \mathcal{S}\setminus\mathcal{M}$ be a random subset chosen by {Stochastic-Greedy} at the $i$-th iteration, it modifies Equation~\ref{eq:gain_genpam_build} to the following:
\begin{equation}\label{eq:gain_genpam_sg}
        g_i = \sum_{j \in \mathcal{T}} q_j \max \{D_j - d(j, i), 0\},\;  \forall i \in \mathcal{R}.
    \end{equation}
We note, for each point $i \in \mathcal{S}$ , it needs $O(|\mathcal{T}|)$ similarity comparisons to evaluate Equation~\ref{eq:gain_genpam_sg} exactly. 
This  might become impractical for a large target set $\mathcal{T}$, and to get around, one can estimate the value via the Monte-Carlo sampling on $\mathcal{T}$ (according to the distribution $\mathbf{q}$).
However, not all the elements in $\mathcal{R}$ are equally strong contenders for being the next prototype. Hence, it is desired to allocate more similarity comparisons to differentiate between contentious elements of $\mathcal{R}$ which are close to the local optimal, while saving them on the easily distinguishable sub-optimal elements. We mitigate this issue by identifying an approximate solution of Equation~\ref{eq:gain_genpam_sg} via multi-armed bandits (MAB)~\citep{bib:Berry+F:1985} that we discuss next.

\subsection{Using Multi-Armed Bandits (MAB) in Prototype Selection}\label{subsec:mab}
\paragraph{Background on Multi-Armed Bandits.} Multi-armed bandits~\citep{bib:Berry+F:1985} is a popular abstraction of sequential decision making under the uncertainty. An \emph{arm} of a bandit represents a decision, while \emph{pull} of an arm represents taking decision corresponding to that arm. Further, we assume each arm has a probability distribution associated with it, and when pulled, a real-valued reward is generated in i.i.d. fashion from the underlying probability-distribution {called reward-distribution}. This probability distribution is called reward distribution of that particular arm, and is unknown to the experimenter. For simplicity, we assume the reward-distribution for each arm is supported on the interval $[0, 1]$. Under this scenario, it is an interesting problem to identify the best arm incurring  minimal number of total  samples. 


To put formally, we assume $\mathcal{A}$ be the given set of $n$ arms with $\mu_a$ being the expected mean reward of arm $a \in \mathcal{A}$. For simplicity, we assume, $\mu_{a_1} \geq \mu_{a_2} \geq \cdots \mu_{a_n}\geq 0$. The problem of best arm identification is defined as identifying the arm $a$ which has the highest expected reward $\mu_a$.
As the underlying reward distribution of the arms are unknown, the only way to compare the arms based on their expected reward is to sample them sufficiently.
Hence, we seek an approximate probabilistic solution and modify the best-arm identification problem as follows. For a given tolerance $\nu \in [0, 1]$, we call an arm $a$ to be \emph{$(\nu, m)$-optimal} (for $m \leq n$) if $\mu_a \geq \mu_{a_m}-\nu$. 
For $m=1$, it is called problem of best-arm identification to identify an $(\nu, 1)$-optimal arm.
 When $ 1 \leq m \leq |\mathcal{A}|$, the problem generalizes to identification of the best subset of size $m$~\citep{bib:Kalyanakrishnan+TAS:2012}, wherein the objective is to identify any $m$ $(\nu, m)$-optimal arms.
 Inheriting notations from~\citet{bib:RoyChaudhuri+K:2019}, let us denote the set of all $(\nu, m)$-optimal arms by $\mathcal{TOP}_m(\nu) \defeq \{a: \mu_a \geq \mu_{a_m}-\nu\}$\footnote{Given a set of arms $\mathcal{A}$ and an $\nu \in (0, 1]$, there might be more than $m$ arms in $\mathcal{TOP}_m(\nu)$.}. 
 Then, an algorithm is said to solve best-subset identification problem if presented with a set of arms $\mathcal{A}$, size of the output set $m$, a tolerance $\nu \in (0, 1)$, and an error probability bound $\delta \in (0, 1)$, it stops with probability 1 after a finite number of steps, and output an $m$-sized subset $Q_{(m)} \in \mathcal{A}$, such that $\Pr\{Q_{(m)} \subset \mathcal{TOP}_m(\nu)\} \geq 1 -\delta$. {The efficiency of solving a best-subset identification lies in incurring as low number of samples as possible.}

\paragraph{Applying MAB to Approximate Equation~\ref{eq:gain_genpam_sg}.} We have seen in Section~\ref{subsec:approx_greedy} that evaluating Equation~\ref{eq:gain_genpam_sg} is intensive in terms of the number of similarity comparisons for a large target set $\mathcal{T}$. 
During every iteration $\mathcal{M}$ remains fixed, so is the dissimilarity $D_j$ between a point $j \in \mathcal{T}$ and the closest object in $\mathcal{M}$. Therefore, if we sample from $\mathcal{T}$ using the probability distribution $\mathbf{q}$, by normalization assumption of the dissimilarity measure, for every element $i \in \mathcal{S}$ we shall get a probability distribution over the  $[0, 1]$, with the true mean being given by Equation~\ref{eq:gain_genpam_build}. Hence, we can treat elements of the source set $\mathcal{S}$ as arms of a MAB instance. Thus, at each round, after we select a random subset $\mathcal{R} \subset \mathcal{S}$, the problem reduces to identification of the best subset of size $r$. Based on this idea, next, we propose an algorithm, and analyze its correctness, and upper bound the incurred number of similarity comparisons in the worst-case. In this paper, we confine to $r=1$, i.e., we select one element per iteration.




The best-arm identification has attracted a lot of attention over the years~\citep{bib:Even-Dar+MM:2002,bib:Even-Dar+MM:2006,bib:Karnin+KS:2013,bib:Jamieson+MNB:2014}. Among an array of algorithms, {Median Elimination}~\citep{bib:Even-Dar+MM:2002} is well-known due to its simplicity, and for incurring a number of samples that is within a constant factor of the lower bound~\citep{bib:Mannor+T:2004}. Recently, ~\citet{bib:Hassidim+KS:2020} have proposed the approximate best arm ({ABA}) algorithm for identification of an $(\nu, 1)$-optimal arm, and have proven its efficiency in sample-complexity compared to the existing algorithms like {Median Elimination}~\citep{bib:Even-Dar+MM:2002}.
We restate Theorem~1 from \citet{bib:Hassidim+KS:2020} on upperbounding the sample-complexity of {ABA}.

\begin{lemma}[Restatement of {\citep[Theorem~1]{bib:Hassidim+KS:2020}}]\label{lem:aba_sc}
Suppose, given a set of arms $\mathcal{A}$, the underlying reward distribution of its each arm is supported on $[0, 1]$. Then, given an $\nu \in (0, 1)$, and $\delta \in (0, 0.5)$, {ABA} initialized with $\alpha=1-e^{-1}$ returns an $(\nu, 1)$-optimal arm with probability at least $1-\delta$ incurring no more than $18\frac{|\mathcal{A}|}{\nu^2}\ln\frac{1}{\delta}$ samples.
\end{lemma}
Now, we are ready to introduce an MAB-based prototype selection algorithm, {\ouralgo}, and present its analysis in the next section.

\subsection{Efficient Prototype Selection by Stochastic Greedy Search and MAB}


We propose a meta algorithm {\ouralgo} in Algorithm~\ref{alg:our_algo} that at each iteration selects a random subset $\mathcal{R} \subset \mathcal{S}$ (like {Stochastic-Greedy}~\citep{bib:Mirzasoleiman+BKVK:2015}) to reduce the search space, and then applies {ABA} to identify the best element. 
We emulate pulling of an arm via sampling the target set $\mathcal{T}$ according to the distribution $\mathbf{q}$ specified by the problem instance.
To this end, we define the method {Pull} in Algorithm~\ref{alg:pull}, that is internally used by the {ABA} inside \ouralgo{}. 
We note, due to normalization assumption, the objective function $f(\cdot)$ defined in Equation~\ref{eq:spot_eq10} must lie in the range $[0, 1]$. In other words, if for all $i \in \mathcal{S}$, and $j \in \mathcal{T}$, then we can take $C=1$ to make $Z_{i,j} \in [0, 1]$. Therefore, for any $\mathcal{M} \subseteq \mathcal{S}$, $f(\mathcal{M}) \in [0, 1]$. Subsequently, the assumptions to apply ABA remain valid.

\begin{algorithm}[ht]
\caption{{\ouralgo}: Randomized Greedy Prototype Selection with MAB}\label{alg:our_algo}
\DontPrintSemicolon
\SetKwInOut{Input}{Input}
\SetKwInOut{Output}{Output}
\SetKwRepeat{Do}{do}{while}
\SetKwInOut{Assert}{Assert}
\Input{Problem instance $(\mathcal{S}, \mathcal{T}, \mathbf{q}, d, k)$, tolerance $\epsilon\in (0, 1)$, $\nu \in (0, 1-\epsilon -1/e)$, and acceptable error probability $\delta \in (0, 0.05)$.}
\Output{Set of prototype $\mathcal{M} \subset \mathcal{S}$, such that $|\mathcal{M}| = k$}

\Assert{$\forall (y, x) \in \mathcal{T} \times \mathcal{S}, d(y, x) \in [0, 1]$.}

Define a similarity function $sim: \mathcal{T} \times \mathcal{S} \mapsto [0, 1]$ as $sim \defeq 1 - d(y, x), \forall (y, x) \in \mathcal{T} \times \mathcal{S}$.\;

Set $\mathcal{M} = \emptyset$

\While{$|\mathcal{M}| \leq k$}{
    Select a subset $\mathcal{R} \subset \mathcal{S} \setminus \mathcal{M}$ of size $|\mathcal{R}| = \ceil*{\frac{|\mathcal{S}|}{k}\log\frac{1}{\epsilon}}$, using uniform random sampling with replacement.\;
    
    
      Apply {ABA} on $\mathcal{R}$, with tolerance $\frac{\nu}{\revision{k}(1-1/e - \epsilon)}$, error threshold $\frac{\delta}{k}$, and return an arm $a_\text{out}$.
    
    $\mathcal{M} = \mathcal{M} \cup \{a_{\text{out}}\}$.
}
\end{algorithm}

\begin{algorithm}[ht]
\caption{{Pull}}\label{alg:pull}
\SetKwInOut{Input}{Input}
\SetKwInOut{Output}{Output}
\Input{Sampling domain $\mathcal{T}$, sampling distribution $\mathbf{q}$, chosen set of prototypes $\mathcal{M} \subset \mathcal{S}$, an element $i \in \mathcal{S}\setminus \mathcal{M}$, similarity function $sim: \mathcal{T} \times \mathcal{S} \mapsto [0, 1]$,.}
\Output{Contribution value  of $i$.}
Sample an element $j \in \mathcal{T}$ according to sampling distribution $\mathbf{q}$.\;

Return $\max \{D_j - d(j, i), 0\}$, wherein $D_j$ is defined in Equation~\ref{eq:mindist_medoid}.
\end{algorithm}

To derive the approximation guarantee for \ouralgo{}, we inherit the analysis of {Stochastic-Greedy} from \citet[Theorem 1]{bib:Mirzasoleiman+BKVK:2015} and lower bound the probability of choosing an element from the optimal set in Lemma~\ref{lem:sample_prob}. Then, we apply this result to prove the approximation ratio. 


\begin{lemma}
\label{lem:sample_prob}
The expected gain of \ouralgo{} is at least $\frac{1-\epsilon}{k} \sum_{a \in \mathcal{M}^* \setminus \mathcal{M}} \left(\Delta(a|\mathcal{M})- \frac{\nu}{\revision{k}(1-\epsilon-1/e)}\right)$ per iteration, where $\mathcal{M}$ is a current solution. 
\end{lemma}

\begin{proof}
To prove the lemma, let us first lower bound the probability of $\mathcal{R} \cap (\mathcal{M}^*\setminus\mathcal{M}) \neq \emptyset$.
We note,
\begin{align}\label{eq:prob_intersect_1}
    & \Pr\{\mathcal{R} \cap (\mathcal{M}^*\setminus\mathcal{M}) = \emptyset\} & \nonumber\\ 
    & = \left(1-\frac{|\mathcal{M}^* \setminus \mathcal{M}|}{|\mathcal{S}\setminus\mathcal{M}|}\right)^{|\mathcal{R}|} \leq  \exp\left(-|\mathcal{R}|\frac{|\mathcal{M}^* \setminus \mathcal{M}|}{|\mathcal{S}\setminus\mathcal{M}|}\right) \nonumber\\
    & \leq \exp\left(-\frac{|\mathcal{R}|}{|\mathcal{S}|}|\mathcal{M}^* \setminus \mathcal{M}|\right) =  \exp\left(-k\frac{|\mathcal{R}|}{|\mathcal{S}|}\frac{|\mathcal{M}^* \setminus \mathcal{M}|}{k}\right).
\end{align}
Now, using the concavity of $exp(-\frac{k|\mathcal{R}|}{|\mathcal{S}|}x)$ with respect to $x$, and given that $|\mathcal{M}^* \setminus \mathcal{M}| \in \{0, 1, \cdots, k\}$,
we can re-write Equation~\ref{eq:prob_intersect_1} as
\begin{align}\label{eq:prob_intersect_2}
&\implies \Pr\{\mathcal{R} \cap (\mathcal{M}^*\setminus\mathcal{M}) \neq \emptyset\} \nonumber\\
    & \geq \left(1-\exp\left(-k\frac{|\mathcal{R}|}{|\mathcal{S}|}\right)\right)\frac{|\mathcal{M}^* \setminus \mathcal{M}|}{k},  \nonumber\\
    & = \left(1-\exp\left(-\ceil*{\frac{|\mathcal{S}|}{k}\log\frac{1}{\epsilon}}\frac{k}{|\mathcal{S}|}\right)\right)\frac{|\mathcal{M}^* \setminus \mathcal{M}|}{k} \left[\because |\mathcal{R}| = \ceil*{\frac{|\mathcal{S}|}{k}\log\frac{1}{\epsilon}}\; \text{in Algorithm}~\ref{alg:our_algo}\right],  \nonumber\\
    & \geq \left(1-\log\frac{1}{\epsilon}\right)\frac{|\mathcal{M}^* \setminus \mathcal{M}|}{k} \geq (1-\epsilon)\frac{|\mathcal{M}^* \setminus \mathcal{M}|}{k}.
\end{align}

Now, we lower bound the value of $\Delta(a|\mathcal{M})$ by lower-bounding the approximation offered by {ABA}.
Suppose, at some iteration, $\mathcal{M}$ be the already chosen set of prototypes.
Also, let $\hat{a} \defeq \argmax_{a \in \mathcal{R}} \Delta(a|\mathcal{M})$ be the  
 the locally optimal solution, and $a_\text{out}$ be the output of {ABA} inside a single iteration of \ouralgo{}.
Assuming $\mathcal{R} \cap (\mathcal{M}^*\setminus\mathcal{M}) \neq \emptyset$, we note,  the marginal contribution of $\hat{a}$ is at least as the marginal contribution of any  element of $\mathcal{R} \cap (\mathcal{M}^*\setminus\mathcal{M})$. 
Therefore, $\Delta(\hat{a}|\mathcal{M})  \geq  \max_{\mathcal{R} \cap (\mathcal{M}^*\setminus\mathcal{M})} \Delta(\hat{a}|\mathcal{M})$ and
\begin{equation}\label{eq:ex_lb_1}
    \E[\Delta(\hat{a}|\mathcal{M})|\mathcal{M}]  \geq  \Pr\{\mathcal{R} \cap (\mathcal{M}^*\setminus\mathcal{M}) \neq \emptyset\} \cdot \max_{a \in \mathcal{R} \cap (\mathcal{M}^*\setminus\mathcal{M})} \Delta(\hat{a}|\mathcal{M}).
\end{equation}
We note, $\Delta(\hat{a}|\mathcal{M})$ is at least as much as the contribution of an element uniformly chosen at random from $\mathcal{R} \cap (\mathcal{M}^*\setminus\mathcal{M})$. 
Again as $\mathcal{R}$ is chosen uniformly at random from $|\mathcal{S} \setminus \mathcal{M}|$, each element of $\mathcal{M}^*\setminus\mathcal{M}$ is equally likely to belong to $\mathcal{R}$.
Hence, choosing an element uniformly at random from $\mathcal{R} \cap (\mathcal{M}^*\setminus\mathcal{M})$ is equivalent to choosing it uniformly at random from $\mathcal{M}^*\setminus\mathcal{M}$.
Putting together, from Equation~\ref{eq:ex_lb_1} we get
\begin{equation}~\label{eq:ex_lb_2}
\E[\Delta(\hat{a}|\mathcal{M})|\mathcal{M}]  \geq \Pr\{\mathcal{R} \cap (\mathcal{M}^*\setminus\mathcal{M}) \neq \emptyset\} \frac{1}{|\mathcal{M}^*\setminus\mathcal{M}|} \sum_{a \in \mathcal{M}^*\setminus\mathcal{M}} \Delta(a|\mathcal{M}).
\end{equation}
In \ouralgo{}, the arm $a_\text{out}$ returned by the subroutine {ABA} is not necessarily identical to $\hat{a}$, but  $\left(\frac{\nu}{\revision{k}(1-\epsilon-1/e)}, 1\right)$-optimal (in $\mathcal{R}$) with probability at least $1-\delta/k$. Therefore, letting $\nu_0\defeq\frac{\nu}{\revision{k}(1-1/e - \epsilon)}$, $\left(\Delta(\hat{a}|\mathcal{M}) - \Delta(a_\text{out}|\mathcal{M})\right) \leq \nu_0$ holds with probability at least $1-\delta/k$. Now, assuming $a_\text{out}$ is $(\nu_0, 1)$-optimal, and following the argument we used to build Equations~\ref{eq:ex_lb_1} and \ref{eq:ex_lb_2}, we can write

\begingroup
\allowdisplaybreaks
\begin{align}
&\E[\Delta(a_\text{out}|\mathcal{M})|\mathcal{M}] \nonumber\\
& \geq  \Pr\{\mathcal{R} \cap (\mathcal{M}^*\setminus\mathcal{M}) \neq \emptyset\}\left(\max_{a \in \mathcal{R} \cap (\mathcal{M}^*\setminus\mathcal{M})} \Delta(\hat{a}|\mathcal{M})-\nu_0\right)\; [\text{w. p.} \geq 1-\frac{\delta}{k}], \label{eq:ex_lb_approx_1}\\
& \geq \Pr\{\mathcal{R} \cap (\mathcal{M}^*\setminus\mathcal{M}) \neq \emptyset\} \left(\frac{1}{|\mathcal{M}^*\setminus\mathcal{M}|} \sum_{a \in \mathcal{M}^*\setminus\mathcal{M}}( \Delta(a|\mathcal{M})- \nu_0)\right), \nonumber\\
& \geq \Pr\{\mathcal{R} \cap (\mathcal{M}^*\setminus\mathcal{M}) \neq \emptyset\} \frac{1}{|\mathcal{M}^*\setminus\mathcal{M}|} \sum_{a \in \mathcal{M}^*\setminus\mathcal{M}} \left(\Delta(a|\mathcal{M})-\nu_0\right), \nonumber\\
& \geq (1-\epsilon)\frac{|\mathcal{M}^* \setminus \mathcal{M}|}{k} \frac{1}{|\mathcal{M}^*\setminus\mathcal{M}|}  \sum_{a \in \mathcal{M}^*\setminus\mathcal{M}} (\Delta(a|\mathcal{M})-\nu_0)\;\; [\text{using Equation~\ref{eq:prob_intersect_2}}],\nonumber\\
& = \frac{1-\epsilon}{k} \sum_{a \in \mathcal{M}^* \setminus \mathcal{M}} (\Delta(a|\mathcal{M})-\nu_0).\label{eq:ex_lb_approx_2}
\end{align}
\endgroup
\end{proof}

\begin{theorem}[Correctness of \ouralgo{}]\label{thm:our_algo}
Suppose, given a problem instace $(\mathcal{S}, \mathcal{T}, \mathbf{q}, d, k)$, $\mathcal{M}^*$ is the optimal solution, and for a tolerance $\epsilon\in (0, 1 -\revision{1/e})$, $\revision{\nu > 0 }$, and an acceptable error probability $\delta$ such that $\delta\revision{/k} \in (0, 0.05)$, let $\mathcal{M}_{ALG}$ be the output by \ouralgo{}. Then, $f(\mathcal{M}_{ALG}) \geq \left(1 - e^{-1} - \epsilon \right)f(\mathcal{M}^*) - \nu$ holds in expectation with probability at least $1 -\delta$, where $f(\cdot)$ is defined in Equation~\ref{eq:spot_eq10}. 
\end{theorem}

\begin{proof}
We use Lemma~\ref{lem:sample_prob} to prove Theorem~\ref{thm:our_algo}. Let $\mathcal{M}^i = \{a_\text{out}^1,\cdots, a_\text{out}^i\}$ be the solution obtained by \ouralgo{} at the end of $i$-th iteration, and $\nu_0 = \frac{\nu}{\revision{k}(1-\epsilon-1/e)}$. 
Then,
\begin{align}
\E[f(\mathcal{M}^{i+1}) - f(\mathcal{M}^i)| \mathcal{M}^{i}] & = \E[\Delta(a_\text{out}^{i+1}|\mathcal{M}^i)|\mathcal{M}^i],\nonumber\\
& \geq \frac{1-\epsilon}{k} \sum_{a \in \mathcal{M}^* \setminus \mathcal{M}} \left(\Delta(a|\mathcal{M}^i) - \nu_0\right) \text{[using Lemma~\ref{lem:sample_prob}]}, \nonumber\\
& \geq \frac{1-\epsilon}{k}  \left(\Delta(\mathcal{M}^*|\mathcal{M}^i) - \revision{k}\nu_0\right)\; [\text{by submodularity of } f(\cdot) ],\nonumber\\
& \geq \frac{1-\epsilon}{k} (f(\mathcal{M}^*) - f(\mathcal{M}^i) - \revision{k}\nu_0).\nonumber
\end{align}
We note, $f(\mathcal{M}^{i})$ is a random variable, due to the randomness introduced by random sampling in the source set $\mathcal{S}$ to select $\mathcal{R}$, and also by the MAB subroutine {ABA}. By taking expectation over  $f(\mathcal{M}^{i})$,
$$\E[f(\mathcal{M}^{i+1}) - f(\mathcal{M}^i)] \geq \frac{1-\epsilon}{k} \E[f(\mathcal{M}^*) - f(\mathcal{M}^i) - \revision{k}\nu_0]. $$
Now, by induction,
\begin{align}\label{eq:approx_final}
    \E[f(\mathcal{M}^k)] & \geq \left(1-\left(\frac{1-\epsilon}{k}\right)^k\right)(f(\mathcal{M}^*)- \revision{k}\nu_0), \nonumber\\
    & \geq \left(1-\left(\exp{(-(1-\epsilon))}\right)\right)(f(\mathcal{M}^*)- \revision{k}\nu_0),\nonumber\\
    & \geq \left(1 - e^{-1} - \epsilon\right)(f(\mathcal{M}^*)- \revision{k}\nu_0) = \left(1 - e^{-1} - \epsilon\right)f(\mathcal{M}^*)-\nu.
\end{align}
The last step to prove Theorem~\ref{thm:our_algo} is proving the correctness of Equation~\ref{eq:approx_final}. We note, Equation~\ref{eq:approx_final} is valid only if Equation~\ref{eq:ex_lb_approx_2} holds for every iteration from $1$ to $k$. Further, the correctness of Equation~\ref{eq:ex_lb_approx_2} is dependent on the Equation~\ref{eq:ex_lb_approx_1}. Now, we note, the arm $a_\text{out}$ returned by {ABA} may not be $(\nu_0, 1)$-optimal with probability at most $\delta/k$, and  Equation~\ref{eq:ex_lb_approx_1} fails to hold. However, the probability at in any of the $k$ iterations, $a_\text{out}$ is not $(\nu_0, 1)$-optimal is at most $\sum_{i=1}^k (\delta/k) = \delta$. Hence, Equation~\ref{eq:approx_final} holds with probability at least $1-\delta$. This completes the proof of  Theorem~\ref{thm:our_algo}.
\end{proof}

We now present the computational complexity result of {\ouralgo}. 
\begin{theorem}[Upper bound on the number of similarity comparisons]\label{thm:disq_coplexity}
Given a problem instance $(\mathcal{S}, \mathcal{T}, \mathbf{q}, d, k)$, $\mathcal{M}^*$ is the optimal solution, and for a tolerance $\epsilon\in (0,  \revision{1-1/e})$, $\nu  >0$, and an acceptable error probability $\delta$ such that $\delta\revision{/k} \in (0, 0.05)$,  the number of similarity comparisons incurred by {\ouralgo}  is in $O\left(\revision{k^3}|\mathcal{S}|\left(\frac{\nu}{1- \epsilon -1/e}\right)^{-2}\log\frac{1}{\epsilon}\log\frac{k}{\delta}\right)$.
\end{theorem}
\begin{proof}
As Lemma~\ref{lem:aba_sc} presents, treating $\mathcal{R}$ as the set of arms, and letting $\nu_0 = \frac{\nu}{\revision{k}(1-e^{-1}-\epsilon)}$, {ABA} incurs at most $\frac{18|\mathcal{R}|}{\nu_0^2}\log\frac{1}{\delta}$ samples to identify an $(\nu_0, 1)$-optimal arm from $\mathcal{R}$. 
Further, at  iteration $i$, there are $(i-1)$ elements in the set $\mathcal{M}$. Hence, for each sampled element from $\mathcal{T}$, it takes $O(|\mathcal{M}|)$ or $i-1$ similarity comparison to compute $D_j$ (given by Equation~\ref{eq:mindist_medoid}). 
Since we pass $\frac{\delta}{k}$ as the allowed error probability for each call to \textsc{ABA}, the total number of similarity comparisons is given by 
\begin{align}
    & \sum_{i=1}^k (i-1) |\mathcal{R}| \times 18 \nu_0^{-2}\log\frac{k}{\delta} =  \sum_{i=1}^k (i-1) \frac{|\mathcal{S}|}{k}\log\frac{1}{\epsilon} \times 18 \nu_0^{-2}\log\frac{k}{\delta}\;\ \left[\because |\mathcal{R}| = \frac{|\mathcal{S}|}{k}\log\frac{1}{\epsilon}\right] \nonumber\\
    & < 9\revision{k^3}|\mathcal{S}|\left(\frac{\nu}{1- \epsilon -1/e}\right)^{-2}\log\frac{1}{\epsilon}\log\frac{k}{\delta}.\nonumber
\end{align}
\end{proof}

It should be noted that the above approximation guarantee and computational complexity analysis of {\ouralgo} also holds when $\T=\S$, i.e., the $k$-medoids clustering setting. Hence, {\ouralgo} scales linearly in $\revision{k^3}|\S|$ for the $k$-medoids clustering problem. 

\subsection{Practical Considerations for Implementing {{\ouralgo}}}\label{subsec:impl_ouralgo}

It is important to note that the number of samples incurred by {ABA}~\citep{bib:Hassidim+KS:2020}
is problem-independent, that is given any fixed-sized set of arms $\mathcal{A}$, for a fixed $\epsilon, \delta \in (0, 0.05)$ the number of samples incurred will be the same (as given by Lemma~\ref{lem:aba_sc}). Hence, it is a non-adaptive algorithm as the number of incurred samples does not depend on the means of the arms. 
However, in practice, it is common to handle set of arms where not all arms have their mean very close to each other.
Hence, it is more efficient to use adaptive algorithms like {KL-LUCB}~\citep{bib:Kaufmann+K:2013} that optimizes the number of samples by taking advantage of the difference between mean of the arms.
At each iteration, {KL-LUCB} judicially selects two arms to sample and the algorithm stops if the confidence-intervals of the arms {crosses a threshold that is adaptively determined by the algorithm}. We note that despite {KL-LUCB} being more sample-efficient than non-adaptive algorithms, complying with it still requires a number of samples that might be too high for a medium-sized datasets. However, in practice, seldom we encounter such pathological cases, and hence, we can use a heuristic to make the algorithm incur lesser number of similarity comparisons by stopping early. We make use of the early-stop heuristic inside {KL-LUCB} for our experiments. 
The early stopping makes use of an optimistic threshold that is easier to meet than the stopping criterion set by {KL-LUCB}.

\section{Experiments}
\label{sec:experiment}

In this section, we show the benefit of {\ouralgo} over SPOTgreedy \citep{bib:Gurumoorthy+JM:2021} in obtaining a good trade-off between the number of distance queries needed and the generalization performance of the obtained prototypes. 

We consider the MNIST~\citep{lecun98a} dataset, which is a collection of 60K images of hand-written digits from 0 to 9, each of size $28 \times 28$ pixels. There are two subsets  mnist\textunderscore train, and mnist\textunderscore test consisting of 50K and 10K images, respectively. Following \cite{bib:Gurumoorthy+JM:2021}, we sample 5K points uniformly at random from mnist\textunderscore test and create the source set $\mathcal{S}$. We create the target set $\mathcal{T}$ from mnist\textunderscore test as follows. First we note, the population of label $5$ is 5421, and it is the least among all the labels. 
Therefore, to create a target set with skew $\theta\%$ we take all the elements of label 5 from mnist\textunderscore train and we uniformly sample its remaining elements leading to a $\mathcal{T}$ of size $(542100/\theta)$. Thus, for $\theta=10$, the target set $\mathcal{T}$ is completely balanced, while for $\theta=100$ it consists elements only from label 5. We conduct different experiments by varying $\theta \in \{10, 20, 50, 70, 100\}$, and $k \in \{100, 200, 500\}$.
For \ouralgo{}, we have used $\epsilon \in \{0.2, 0.4\}$, and $\revision{\nu_0} \in \{0.05, 0.09\}$ (\revision{at each iteration of {\ouralgo}}). The results are averaged over ten randomized runs. {We use the Euclidean distance as the pair-wise dissimilarity measure $d$ between points and normalize the distances to be in $[0,1]$.}

In our experiments, we consider the setting in which both SPOTgreedy and \ouralgo{} select only one element at each iteration, i.e., $r = 1$. We also ensure that both {\ouralgo} and SPOTgreedy compute the similarity comparisons on the fly (i.e., they do not memorize the computed similarity comparisons between data points). 

The results are shown in Figure~\ref{fig:comp_distq_obj_acc_sk_10_eps02}. 
For $\mathcal{T}$ with 10\% skew (i.e., $\theta=10$), Figure~\ref{fig:comp_distq_obj_acc_sk_10_eps02}(a) shows that \ouralgo{} (for $\epsilon=0.2$) incurs as little as $1/100$-th of the number of similarity comparisons made by SPOTgreedy; however as Figure~\ref{fig:comp_distq_obj_acc_sk_10_eps02}(a) shows the drop in objective value is within only~2\%.

We note that the number of distance computations for SPOTgreedy is minimum if it is allowed to memorize all the $|\mathcal{S} \times \mathcal{T}|$ pair-wise similarity values. Let us call this implementation of SPOTgreedy where it is allowed to memorize all these values as SPOT\textunderscore M. We also compare {\ouralgo} with SPOT\textunderscore M. Figure~\ref{fig:comp_distq_obj_acc_sk_10_eps02}(a) shows the number of incurred distance queries by the proposed  \ouralgo{} is far less as compared to SPOT\textunderscore M.

In practice, it is very common that value of the objective is not the final thing that an experimenter seeks.
A more sensible way to compare the results is via the accuracy of the selected prototypes.
Thanks to Lemma~\ref{lem:spot_build_pam}, we can now inherit the barycentric projection method from the theory of optimal transport. Following~\citet{bib:Gurumoorthy+JM:2021}, the accuracy of the selected prototypes is measured by using the barycentric mapping for both SPOTgreedy and \ouralgo{}. As depicted by Figure~\ref{fig:comp_distq_obj_acc_sk_10_eps02}(c), \ouralgo{} achieves a very close accuracy compared to SPOTgreedy. Detailed comparisons between these two algorithms for different values of skew ($\theta$), $\epsilon$, and $\nu$ are in \conditional{Appendix \ref{app:experiment}}{Appendix C of the extended version of this paper \citep{chaudhuri22a}}.


In \conditional{Appendix~\ref{app:experiment_cifar}}{Appendix~D of the extended version of this paper \citep{chaudhuri22a}},
we present experiments on an additional dataset.

\begin{figure}[t]
    \centering
    \subfigure{\includegraphics[height=3cm,width=0.32\linewidth]{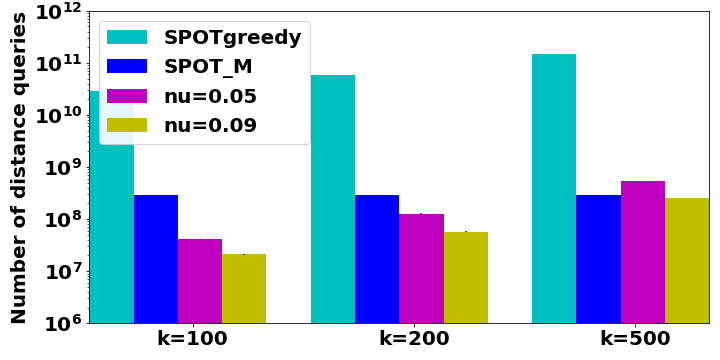}\label{subfig:distq_comp_sk_10_eps02}}
    \subfigure{\includegraphics[height=3cm,width=0.32\linewidth]{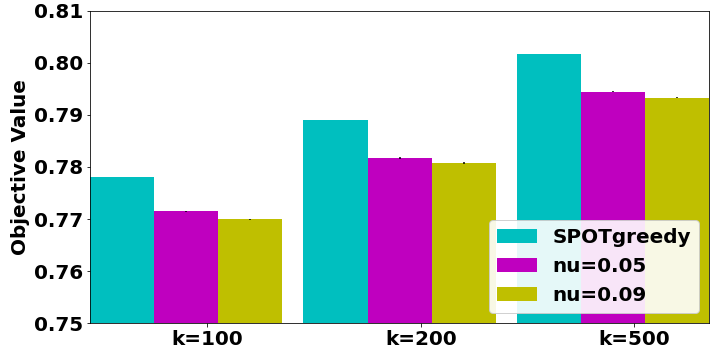}\label{subfig:obj_comp_sk_10_eps02}}
    \subfigure{\includegraphics[height=3cm,width=0.32\linewidth]{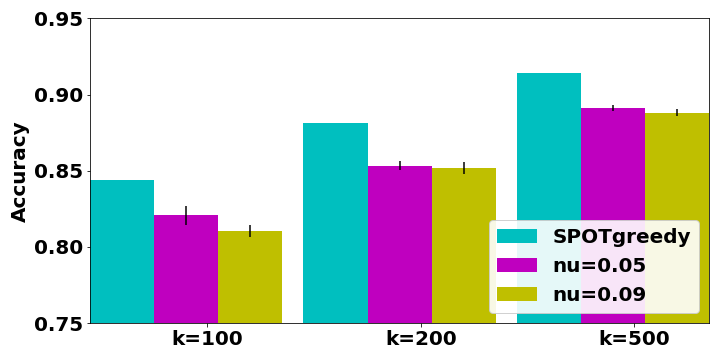}\label{subfig:acc_comp_sk_10_eps02}}
    \caption{Comparison of SPOTgreedy and \ouralgo{} (with $\epsilon=0.2$) on $\mathcal{T}$ with skew = 10\% in terms of (a) number of similarity comparisons; (b) objective value; and (c) generalization performance. SPOT\textunderscore M in Figure~\ref{fig:comp_distq_obj_acc_sk_10_eps02}(a) is an implementation of SPOTgreedy with all the $|\mathcal{S} \times \mathcal{T}|$ pairwise-similarity values memorized. Our proposed {\ouralgo} with different $\nu$ values obtain good accuracy at much lower number of similarity comparisons for SPOTgreedy. Compared to SPOT\textunderscore M, which memorizes all the $|\mathcal{S} \times \mathcal{T}|$ pair-wise similarity values, we see the benefit of {\ouralgo} (that never memorizes the similarity values). \revision{Here, $\nu$ parameter corresponds to per iteration $\nu_0$ of {\ouralgo}.}}\label{fig:comp_distq_obj_acc_sk_10_eps02}
\end{figure}

\section{Conclusion and Future Work}\label{sec:future_work}
We have proposed a novel unsupervised algorithm {\ouralgo} for the prototype selection problem, which offers a good trade-off between the number of similarity comparisons needed and the quality of prototypes obtained. The key idea is to generalize the popular clustering method PAM~\citep{kmediods} for the prototype selection problem. We introduce two strategies for solving the \texttt{BUILD} step of generalized PAM: first, use of random subset selection (to reduce the search space), and second, using multi-armed bandits based identification of an approximately best candidate. 
We also provide an approximation guarantee of the prototypical set obtained by the proposed {\ouralgo} algorithm. This allows the proposed algorithm {\ouralgo} to have an upper bound on the number of similarity comparisons that is independent of target set size $|\T|$ and that scales linearly in source set size $|\S|$. For the case of $\T = \S$, a salient observation is that {\ouralgo} approximates the \texttt{BUILD} solution of PAM for the $k$-medoids clustering problem in $O(\revision{k^3}|\S|)$ complexity. 


Currently, our analysis relies critically on the application of \citep[Theorem~1]{bib:Hassidim+KS:2020} for deriving the bounds. It would be interesting to see whether we can tighten the analysis for better bounds. Another research direction would be to analyze the \texttt{SWAP} step of generalized PAM within our framework in order to refine the prototypes obtained. Equally intriguing would be to explore distributed MAB approaches~\citep{bib:Mahadik+WLS:2020,li2016collaborative} for the prototype selection problem.

\section*{Acknowledgment}
We would like to thank Karthik Gurumoorthy for his insightful comments, discussion, and suggestions, especially on highlighting our bounds for the $k$-medoids problem. We would also like to thank the reviewers for their feedback.

%% file: appendix.tex
\section{Generalized PAM Algorithm for Prototype Selection}
\label{app:pam}
\begin{itemize}
    \item As defined in Equation~\ref{eq:mindist_medoid}, let $D_j$ be the dissimilarity between a point $j$ and the closest object in $\mathcal{M}$, i.e. 
    \begin{equation}
        D_j = \min_{i \in \mathcal{M}} \{d(j, i)\}.
    \end{equation}
    
    \item Let, $E_j$ be the be the dissimilarity between a point $j$ and the \emph{second} closest object in $\mathcal{S}$, i.e., $E_j = \min_{s \in \mathcal{S}\setminus \{\argmin_{c \in \mathcal{M}} \{d(j, c)\}\} } \{d(j, s)\}$.
\end{itemize}

\begin{algorithm}[ht]
\caption{Generalized PAM algorithm for prototype selection}\label{alg:genpam_full}
\DontPrintSemicolon
\SetKwInOut{Input}{Input}
\SetKwInOut{Output}{Output}
\Input{Problem instance given by $(\mathcal{S}, \mathcal{T},\mathbf{q}, d, k)$, and a pair of positive integers $r, l \in \{1, \cdots, k\}$.}
\Output{A set $\mathcal{M}$, such that $\mathcal{M} = k$}
\textbf{Initialization:} For each $j \in \mathcal{T}$, set $D_j = \infty$.\;

\Begin(\textbf{\texttt{BUILD} Step:}){
 In this step we apply greedy strategy to choose an initial set of  $k$ points, in other words, an initial set of medoids.\; 
$\mathcal{M} = \emptyset$\;
\While{$|\mathcal{M}| < k$}{
    Define vector $g$ with entries
    \begin{equation}\label{eq:gain_genpam}
        g_i = \sum_{j \in \mathcal{T}} q_j \max \{D_j - d(j, i), 0\},\;  \forall i \in \mathcal{S} \setminus \mathcal{M}.
    \end{equation}
    \uIf(\tcp*[h]{no meaningful medoid to add}){$\forall i \in \mathcal{S}, g_i = 0$}{
    $Q \defeq$ Choose $r$ elements from  $\mathcal{S}\setminus\mathcal{M}$ at random \; 
    }\Else{
    $Q \defeq$ Set of indices of top $r$ largest \emph{non-zero} elements in $g = \{g_i\}_{i=1}^{|\{\mathcal{S} \setminus \mathcal{M}\}|}$ \;
    }
    $\mathcal{M} = \mathcal{M} \cup Q$\;
    Update $D_j$, and $E_j$ for all $j \in \mathcal{T}$.\;
}
}
\Begin(\textbf{{\texttt{SWAP}} Step: or refinement step}){
\For{$(i, h) \in \mathcal{S} \times \{\mathcal{S} \setminus \mathcal{M}\}$}{
\begin{equation*}
    \text{compute}\; K_{jih} = 
    \begin{cases}
        \min \{d(j, h) - D_j, 0)\}\; \forall j \in \mathcal{T}, \text{where}\; d(j,i) > D_j\\
        \min \{d(j, h) - E_j, 0)\}\; \forall j \in \mathcal{T}, \text{where}\; d(j,i) = D_j
    \end{cases}
\end{equation*}
$T_{ih} = \sum_{j \in \mathcal{T}} q_j K_{jih}$
}
For $l$ different $h \in \{\mathcal{S} \setminus \mathcal{M}\}$, select pairs $(i, h) \in \mathcal{S} \times \{\mathcal{S} \setminus \mathcal{M}\}$ that minimize $T_{ih}$.
}
\uIf{$T_{ih} < 0$} {
Swap $i$ with $h$.
}\Else(\tcp*[h]{value of the objective cannot
be decreased}){
\Return $\mathcal{M}$
}
\end{algorithm}

\section{Proof of Lemma \ref{lem:spot_build_pam}}
\label{subsec:spot_restate} \label{app:spot_build_pam}

 SPOTgreedy at its core chooses medoids by solving the following optimisation problem given by Equation~\ref{eq:spot_eq10}.
Assuming $f(\emptyset) = 0$, defined in Equation \ref{eq:spot_eq10}, and noting that due to $d: \mathcal{T} \times \mathcal{S} \mapsto [0, 1]$, $\beta_i$ (in Algorithm~\ref{alg:spot}) lies within $[0, 1]$
we restate SPOTgreedy \citep{bib:Gurumoorthy+JM:2021} in Algorithm~\ref{alg:spot}.

\begin{algorithm}[ht]
\DontPrintSemicolon
\SetKwInOut{Input}{Input}
\SetKwInOut{Output}{Output}
\caption{SPOTgreedy~\citep{bib:Gurumoorthy+JM:2021}}\label{alg:spot}
\Input{Problem instance given by $(\mathcal{S}, \mathcal{T}, \mathbf{q}, d, k)$, and a positive integer $r \leq k$.}
\Output{A set $\mathcal{M}$, such that $|\mathcal{M}| = k$}
\While{$|\mathcal{M}| < k$}{
Define gain vector $\beta$ with entries 
\begin{equation}\label{eq:gain_spot}
 \beta_i = f(\mathcal{M} \cup \{i\}) - f(\mathcal{M}), \forall i \in \mathcal{S} \setminus \mathcal{M},
\end{equation}
where $f:\mathcal{S} \mapsto [0, +\infty)$ is defined in Equation~\ref{eq:spot_eq10}.\;

\uIf(\tcp*[h]{no potential prototype to add}){$\forall i \in \mathcal{S}, \beta_i = 0$}{
        $Q \defeq$ Choose $r$ elements from  $\mathcal{S}\setminus\mathcal{M}$ at random \; \;
    }\Else{
    $Q \defeq$ Set of indices of top $r$ largest \emph{non-zero} elements in $\beta = \{\beta_i\}_{i=1}^{|\mathcal{S} \setminus \mathcal{M}|}$\;
    }
$\mathcal{M} = \mathcal{M} \cup Q$
}
\end{algorithm}


\begin{proof}
Let, at the beginning of some iteration $t$, the set of points already chosen by SPOTgreedy (Algorithm~\ref{alg:spot}) and Algorithm~\ref{alg:genpam_build} be identical, and let it be $\mathcal{M}_t$. We note, such $t$ exists; for example, at beginning of the very first iteration ($t=1$), $\mathcal{M}_t = \{\emptyset\}$. Now, to prove the lemma, it will suffice to prove that, $g_i$ on the LHS of Equation~\ref{eq:gain_genpam_build} is identical to $\beta_i$ on the LHS of Equation~\ref{eq:gain_spot}.

We note, for all $i \in \mathcal{S} \setminus \mathcal{M}$, we have
\begin{align*}
    \beta_i & =  f(\mathcal{M}_t \cup \{i\}) - f(\mathcal{M}_t)  \\
            &  =  \sum\limits_{j \in \mathcal{T}} q_j \max\limits_{i \in \mathcal{M}_t \cup \{i\}} Z_{ji} -  \sum\limits_{j \in \mathcal{T}} q_j \max\limits_{i \in \mathcal{M}_t} Z_{ji} && [\text{using Equation~\ref{eq:spot_eq10}}],\\
            & = \sum\limits_{j \in \mathcal{T}} q_j \max\limits_{i \in \mathcal{M}_t \cup  \{i\}} (C - d(j, i)) -  \sum\limits_{j \in \mathcal{T}} q_j \max\limits_{i \in \mathcal{M}_t} (C - d(j, i))  && [\text{using Equation~\ref{eq:spot_sim}}],\nonumber\\
            & = \sum\limits_{j \in \mathcal{T}} q_j \left(\max\limits_{i \in \mathcal{M}_t \cup  \{i\}} (C - d(j, i)) -  \max\limits_{i \in \mathcal{M}_t} (C - d(j, i))\right), && \text{[changing difference of sums}\nonumber\\[-15pt]
            & && \text{ to sum of differences],}\nonumber\\
            & = \sum\limits_{j \in \mathcal{T}} q_j \left(C + \max\limits_{i \in \mathcal{M}_t \cup  \{i\}}(- d(j, i)) -  C - \max\limits_{i \in \mathcal{M}_t} (- d(j, i))\right)  && [\because \text{$C$ is a constant}],\nonumber\\
            & = \sum\limits_{j \in \mathcal{T}} q_j \left(\max\limits_{i \in \mathcal{M}_t \cup  \{i\}}(- d(j, i)) -  \max\limits_{i \in \mathcal{M}_t} (- d(j, i))\right)  && [\because \text{$C$ is a constant}],\nonumber\\
            & = \sum\limits_{j \in \mathcal{T}} q_j \left(-\min\limits_{i \in \mathcal{M}_t \cup  \{i\}}(d(j, i)) + \min\limits_{i \in \mathcal{M}_t} (d(j, i))\right)  && \text{[exchanging order of }\nonumber\\[-15pt]
            & && \text{$\max$  and `$-$' sign],}\nonumber\\
            & = \sum\limits_{j \in \mathcal{T}} q_j \left(\min\limits_{i \in \mathcal{M}_t} d(j, i) - \min\limits_{i \in \mathcal{M}_t \cup  \{i\}} d(j, i)\right)  && [\text{rearranging the terms}], \nonumber\\
            & = \sum\limits_{j \in \mathcal{T}} q_j \left(D_j - \min\limits_{i \in \mathcal{M}_t \cup  \{i\}} d(j, i)\right)  && \text{[using $\mathcal{M} = \mathcal{M}_t$ at} \nonumber\\[-15pt]
            & && \text{Equation~\ref{eq:mindist_medoid}],}\nonumber\\
            & = \begin{dcases}
                    \sum\limits_{j \in \mathcal{T}} q_j \left(D_j - d(j, i)\right) & \text{if}\; D_j > d(j, i), \nonumber\\
                    0 & \text{otherwise,}
                \end{dcases} \nonumber\\
            & = \sum\limits_{j \in \mathcal{T}} q_j \max\left(D_j - d(j, i), 0\right), \nonumber\\
            & = g_i  && [\text{using Equation~\ref{eq:gain_genpam_build} in }]\nonumber\\
            & && \text{Algorithm~\ref{alg:genpam_build}].}          
\end{align*}
Therefore, for any given problem instance $(\mathcal{S}, \mathcal{T}, \mathbf{q}, d, k)$, for the same value of the input parameter $r$, the output of SPOTgreedy is identical to the output of  Algorithm~\ref{alg:genpam_build}.
\end{proof}

\section{Additional Experiments on MNIST Dataset}\label{app:experiment} 
\subsection{Comparison of performance of {\ouralgo} (with $\epsilon=0.2$) against SPOTgreedy}\label{app:subsec:eps02}
\begin{figure}[H]
    \centering
    \subfigure{\includegraphics[height=3cm,width=0.32\linewidth]{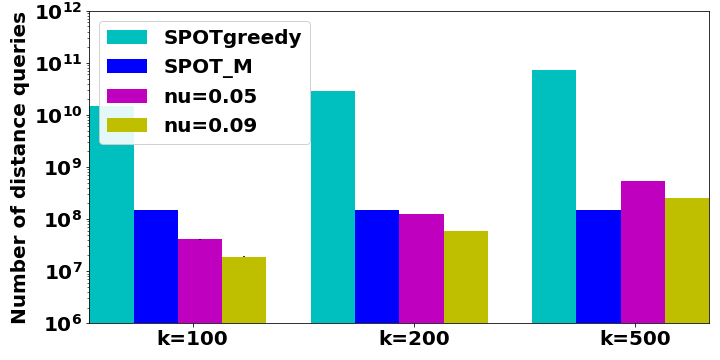}\label{subfig:distq_comp_sk_20_eps02}}
    \subfigure{\includegraphics[height=3cm,width=0.32\linewidth]{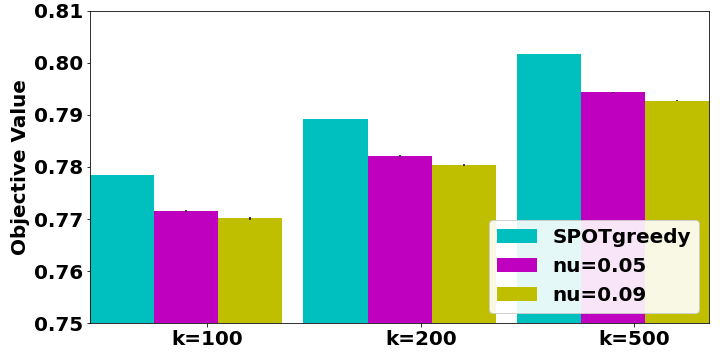}\label{subfig:obj_comp_sk_20_eps02}}
    \subfigure{\includegraphics[height=3cm,width=0.32\linewidth]{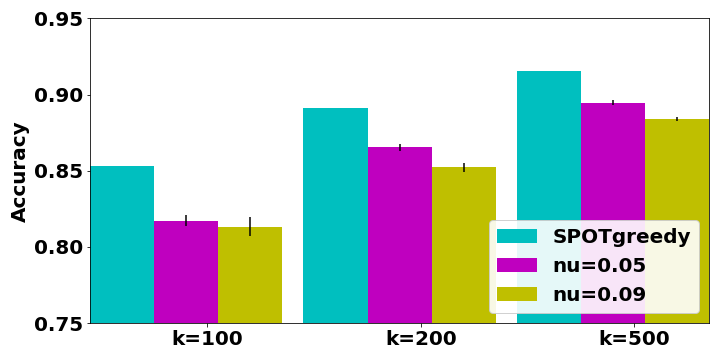}\label{subfig:acc_comp_sk_20_eps02}}
    \caption{Comparison of number of similarity comparisons, objective value, and achieved accuracy by SPOTgreedy and {\ouralgo} (with $\epsilon=0.2$) on Target with skew = 20\%. SPOT\textunderscore M is an implementation of SPOTgreedy with all the $|\mathcal{S} \times \mathcal{T}|$ pairwise-similarity values memorised. For details, see Section~\ref{sec:experiment}.} \label{fig:comp_distq_obj_acc_sk_20_eps02}
\end{figure}

\begin{figure}[H]
\label{fig:comp_distq_obj_acc_sk_50_eps02}
    \centering
    \subfigure{\includegraphics[height=3cm,width=0.32\linewidth]{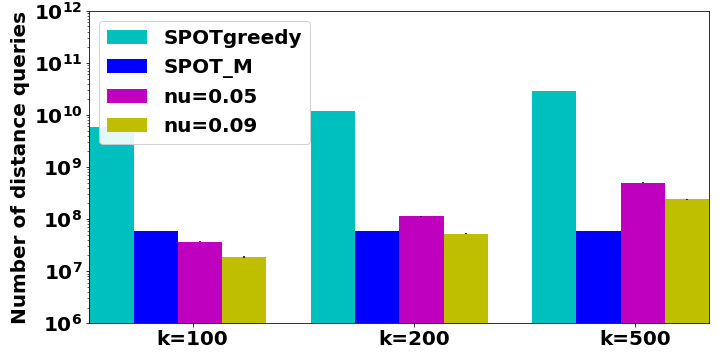}\label{subfig:distq_comp_sk_50_eps02}}
    \subfigure{\includegraphics[height=3cm,width=0.32\linewidth]{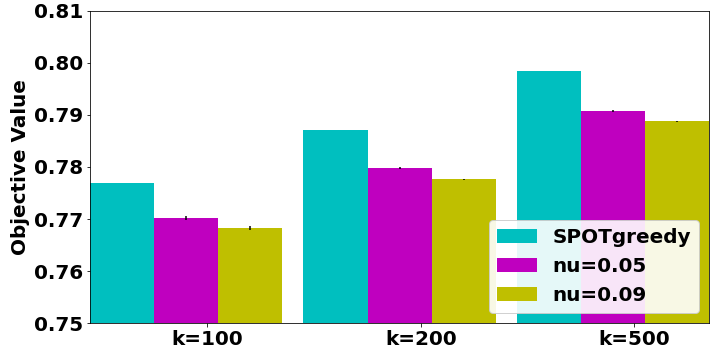}\label{subfig:obj_comp_sk_50_eps02}}
    \subfigure{\includegraphics[height=3cm,width=0.32\linewidth]{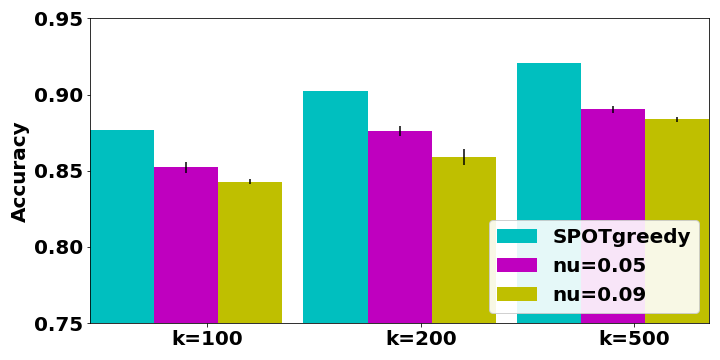}\label{subfig:acc_comp_sk_50_eps02}}
    \caption{Comparison of number of similarity comparisons, objective value, and achieved accuracy by SPOTgreedy and {\ouralgo} (with $\epsilon=0.2$) on Target with skew = 50\%. SPOT\textunderscore M is an implementation of SPOTgreedy with all the $|\mathcal{S} \times \mathcal{T}|$ pairwise-similarity values memorised. For details, see Section~\ref{sec:experiment}.}
\end{figure}

\begin{figure}[H]
\label{fig:comp_distq_obj_acc_sk_70_eps02}
    \centering
    \subfigure{\includegraphics[height=3cm,width=0.32\linewidth]{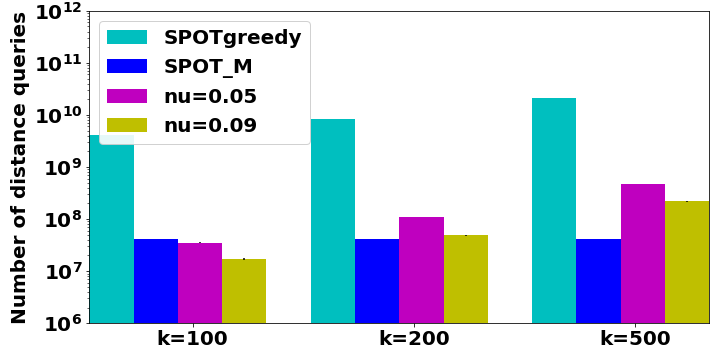}\label{subfig:distq_comp_sk_70_eps02}}
    \subfigure{\includegraphics[height=3cm,width=0.32\linewidth]{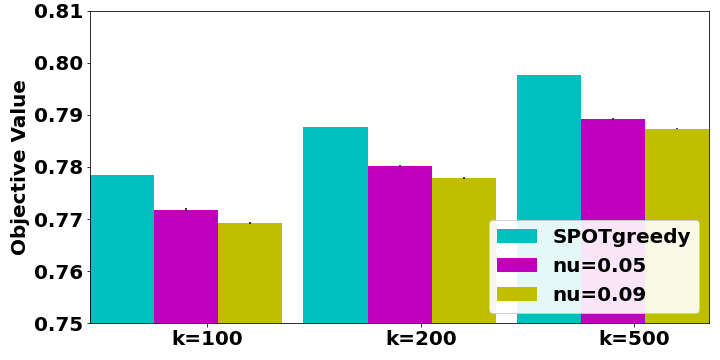}\label{subfig:obj_comp_sk_70_eps02}}
    \subfigure{\includegraphics[height=3cm,width=0.32\linewidth]{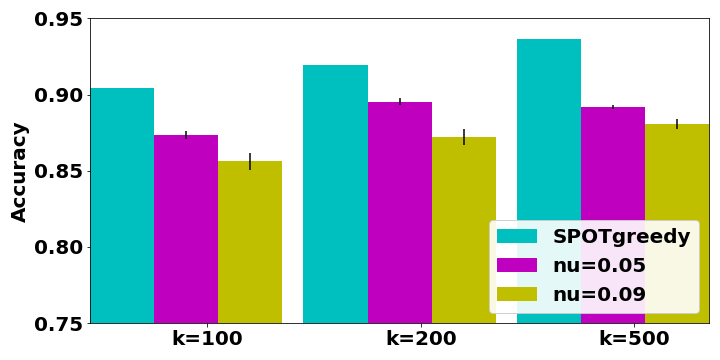}\label{subfig:acc_comp_sk_70_eps02}}
    \caption{Comparison of number of similarity comparisons, objective value, and achieved accuracy by SPOTgreedy and {\ouralgo} (with $\epsilon=0.2$) on Target with skew = 70\%. SPOT\textunderscore M is an implementation of SPOTgreedy with all the $|\mathcal{S} \times \mathcal{T}|$ pairwise-similarity values memorised. For details, see Section~\ref{sec:experiment}.}
\end{figure}

\begin{figure}[H]
\label{fig:comp_distq_obj_acc_sk_100_eps02}
    \centering
    \subfigure{\includegraphics[height=3cm,width=0.32\linewidth]{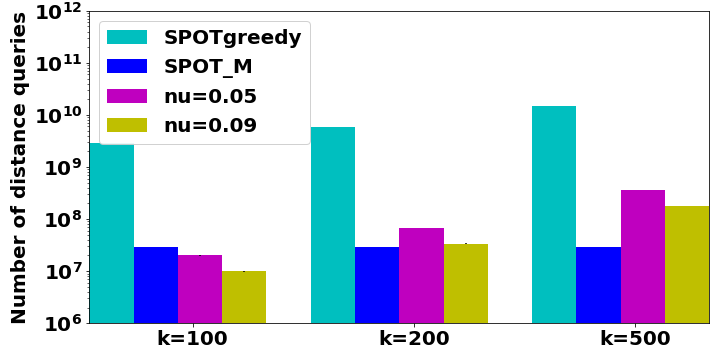}\label{subfig:distq_comp_sk_100_eps02}}
    \subfigure{\includegraphics[height=3cm,width=0.32\linewidth]{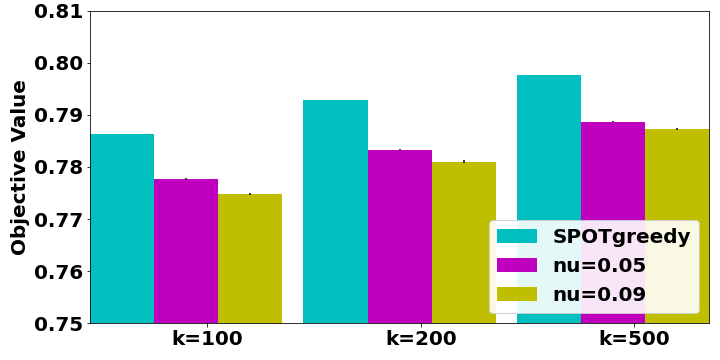}\label{subfig:obj_comp_sk_100_eps02}}
    \subfigure{\includegraphics[height=3cm,width=0.32\linewidth]{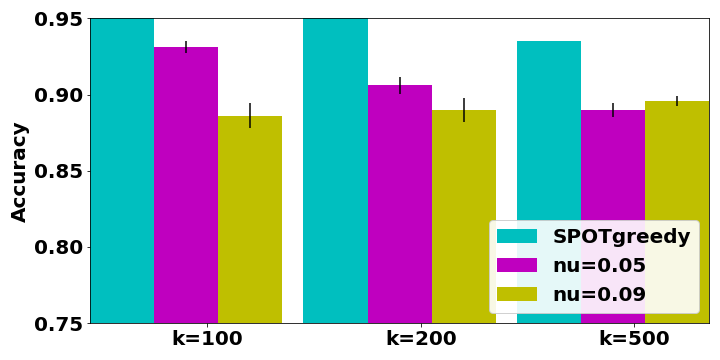}\label{subfig:acc_comp_sk_100_eps02}}
    \caption{Comparison of number of similarity comparisons, objective value, and achieved accuracy by SPOTgreedy and {\ouralgo} (with $\epsilon=0.2$) on Target with skew = 100\%. SPOT\textunderscore M is an implementation of SPOTgreedy with all the $|\mathcal{S} \times \mathcal{T}|$ pairwise-similarity values memorised. For details, see Section~\ref{sec:experiment}.}
\end{figure}

\subsection{Comparison of performance of {\ouralgo{}} (with $\epsilon=0.4$) against SPOTgreedy and SPOT\textunderscore M}\label{app:subsec:eps04}

\begin{figure}[H]
\label{fig:comp_distq_obj_acc_sk_10_eps04}
    \centering
    \subfigure{\includegraphics[height=3cm,width=0.32\linewidth]{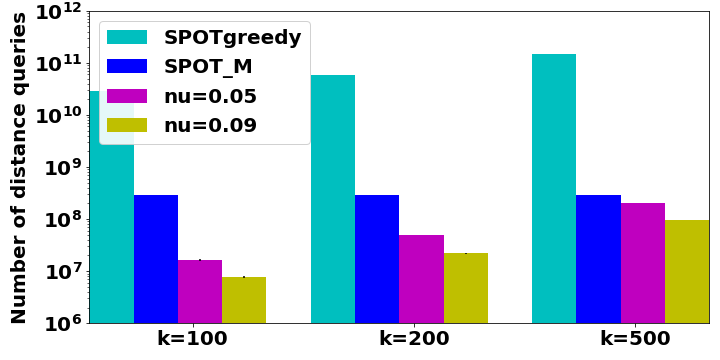}\label{subfig:distq_comp_sk_10_eps04}}
    \subfigure{\includegraphics[height=3cm,width=0.32\linewidth]{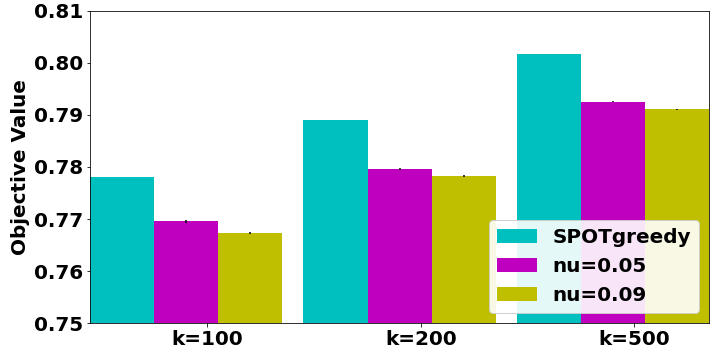}\label{subfig:obj_comp_sk_10_eps04}}
    \subfigure{\includegraphics[height=3cm,width=0.32\linewidth]{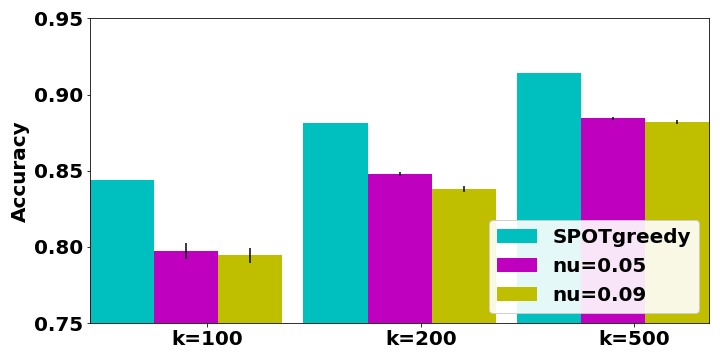}\label{subfig:acc_comp_sk_10_eps04}}
    \caption{Comparison of number of similarity comparisons, objective value, and achieved accuracy by SPOTgreedy and {\ouralgo} (with $\epsilon=0.4$) on Target with skew = 10\%. SPOT\textunderscore M is an implementation of SPOTgreedy with all the $|\mathcal{S} \times \mathcal{T}|$ pairwise-similarity values memorised. For details, see Section~\ref{sec:experiment}.}
\end{figure}

\begin{figure}[H]
\label{fig:comp_distq_obj_acc_sk_20_eps04}
    \centering
    \subfigure{\includegraphics[height=3cm,width=0.32\linewidth]{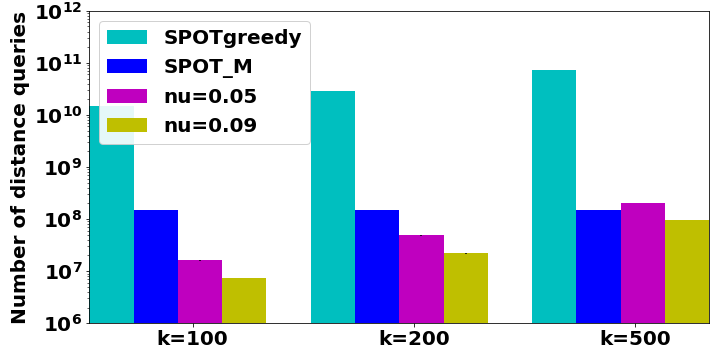}\label{subfig:distq_comp_sk_20_eps04}}
    \subfigure{\includegraphics[height=3cm,width=0.32\linewidth]{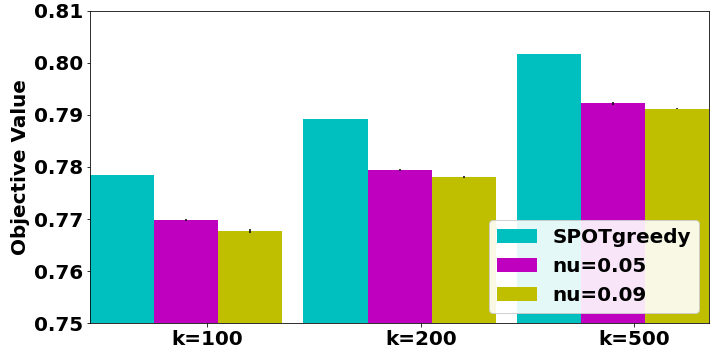}\label{subfig:obj_comp_sk_20_eps04}}
    \subfigure{\includegraphics[height=3cm,width=0.32\linewidth]{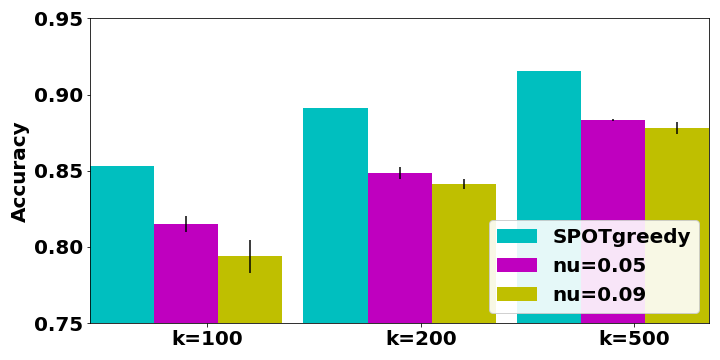}\label{subfig:acc_comp_sk_20_eps04}}
    \caption{Comparison of number of similarity comparisons, objective value, and achieved accuracy by SPOTgreedy and {\ouralgo} (with $\epsilon=0.4$) on Target with skew = 20\%. SPOT\textunderscore M is an implementation of SPOTgreedy with all the $|\mathcal{S} \times \mathcal{T}|$ pairwise-similarity values memorised. For details, see Section~\ref{sec:experiment}.}
\end{figure}

\begin{figure}[H]
\label{fig:comp_distq_obj_acc_sk_50_eps04}
    \centering
    \subfigure{\includegraphics[height=3cm,width=0.32\linewidth]{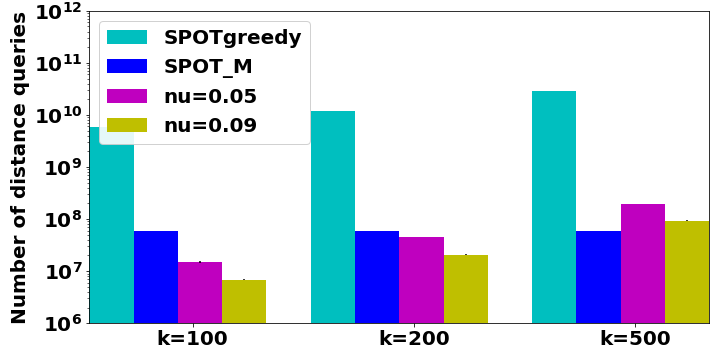}\label{subfig:distq_comp_sk_50_eps04}}
    \subfigure{\includegraphics[height=3cm,width=0.32\linewidth]{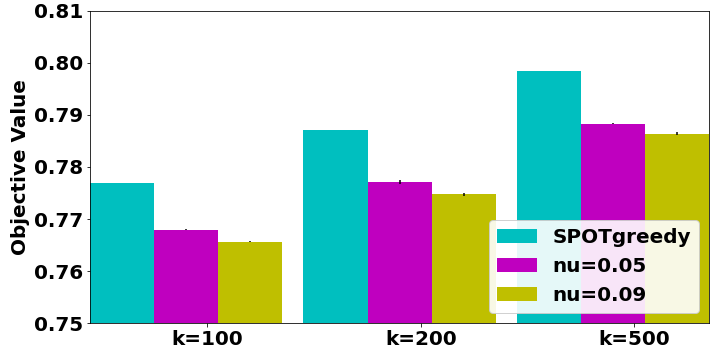}\label{subfig:obj_comp_sk_50_eps04}}
    \subfigure{\includegraphics[height=3cm,width=0.32\linewidth]{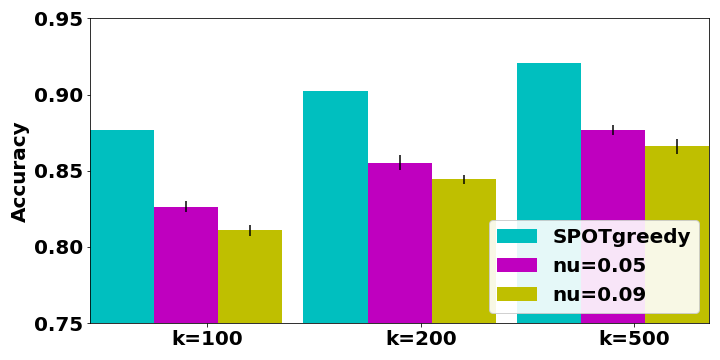}\label{subfig:acc_comp_sk_50_eps04}}
    \caption{Comparison of number of similarity comparisons, objective value, and achieved accuracy by SPOTgreedy and {\ouralgo} (with $\epsilon=0.4$) on Target with skew = 50\%. SPOT\textunderscore M is an implementation of SPOTgreedy with all the $|\mathcal{S} \times \mathcal{T}|$ pairwise-similarity values memorised. For details, see Section~\ref{sec:experiment}.}
\end{figure}

\begin{figure}[H]
\label{fig:comp_distq_obj_acc_sk_70_eps04}
    \centering
    \subfigure{\includegraphics[height=3cm,width=0.32\linewidth]{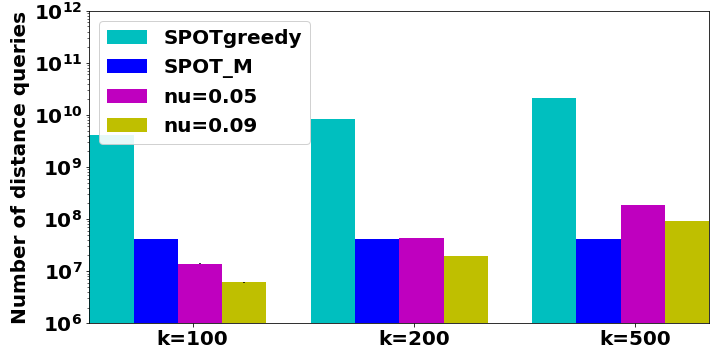}\label{subfig:distq_comp_sk_70_eps04}}
    \subfigure{\includegraphics[height=3cm,width=0.32\linewidth]{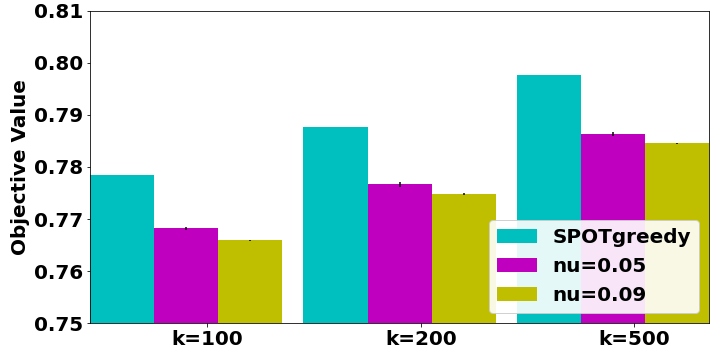}\label{subfig:obj_comp_sk_70_eps04}}
    \subfigure{\includegraphics[height=3cm,width=0.32\linewidth]{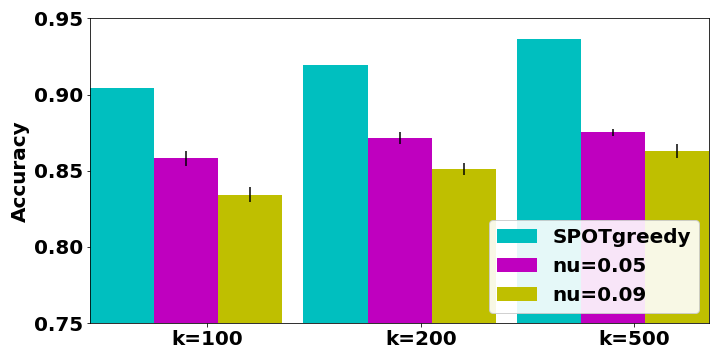}\label{subfig:acc_comp_sk_70_eps04}}
    \caption{Comparison of number of similarity comparisons, objective value, and achieved accuracy by SPOTgreedy and {\ouralgo} (with $\epsilon=0.4$) on Target with skew = 70\%. SPOT\textunderscore M is an implementation of SPOTgreedy with all the $|\mathcal{S} \times \mathcal{T}|$ pairwise-similarity values memorised. For details, see Section~\ref{sec:experiment}.}
\end{figure}

\begin{figure}[H]
\label{fig:comp_distq_obj_acc_sk_100_eps04}
    \centering
    \subfigure{\includegraphics[height=3cm,width=0.32\linewidth]{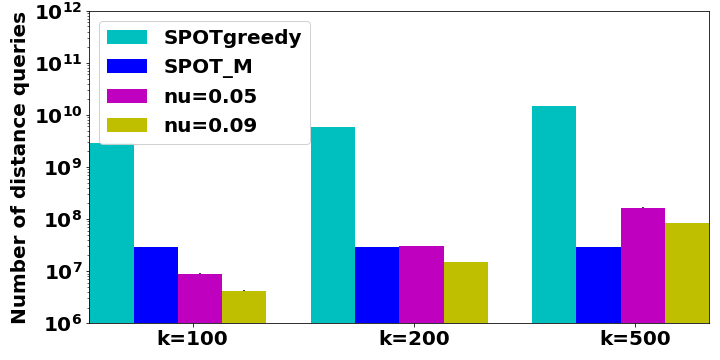}\label{subfig:distq_comp_sk_100_eps04}}
    \subfigure{\includegraphics[height=3cm,width=0.32\linewidth]{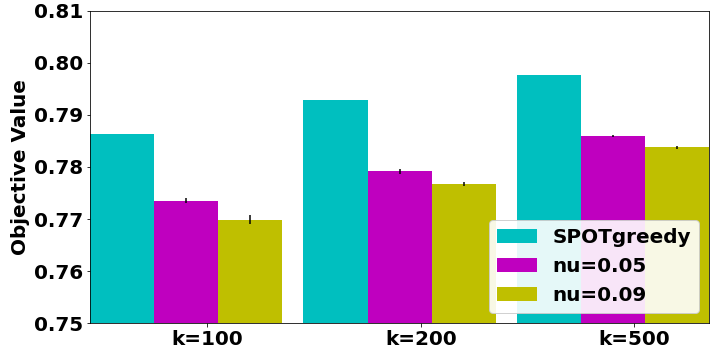}\label{subfig:obj_comp_sk_100_eps04}}
    \subfigure{\includegraphics[height=3cm,width=0.32\linewidth]{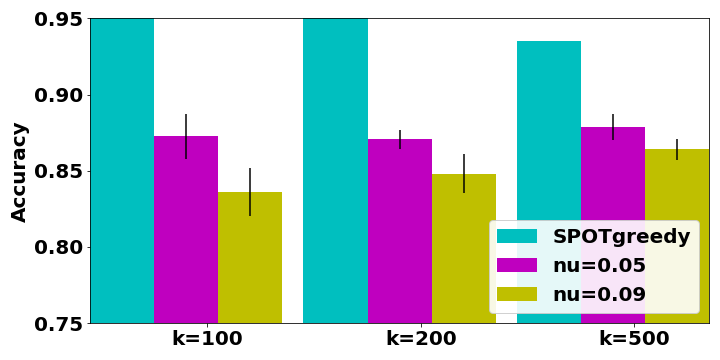}\label{subfig:acc_comp_sk_100_eps04}}
    \caption{Comparison of number of similarity comparisons, objective value, and achieved accuracy by SPOTgreedy and {\ouralgo} (with $\epsilon=0.4$) on Target with skew = 100\%. SPOT\textunderscore M is an implementation of SPOTgreedy with all the $|\mathcal{S} \times \mathcal{T}|$ pairwise-similarity values memorised. For details, see Section~\ref{sec:experiment}.}
\end{figure}

\section{Experiment on CIFAR-10 Dataset}\label{app:experiment_cifar} 

We run experiments on the CIFAR-10 dataset \citep{bib:cifar10}. The dataset consists 60000 color images from $10$classes. Each image is of size $32\times32$ pixels, with $6000$ images per class. The dataset is split into training and test sets consisting of $50\,000$ and $10\,000$ images, respectively. We have assumed the training and the test set as the target and source sets, respectively. For the sake of better representation, we convert these images into gray-scale followed by PCA and choosing the first $161$ dimensions to cover the $95\%$ variance. Just like the MNIST data, here too we use the Euclidean distance (normalized within [0,1]) as the pair-wise dissimilarity measure $d$..

We observe that \ouralgo{} achieves a very close objective value as compared to SPOTgreedy,  but incurs only 1/100-th of the number of similarity queries. 


\begin{figure}[h]
    \centering
    \subfigure{\includegraphics[height=3cm,width=0.32\linewidth]{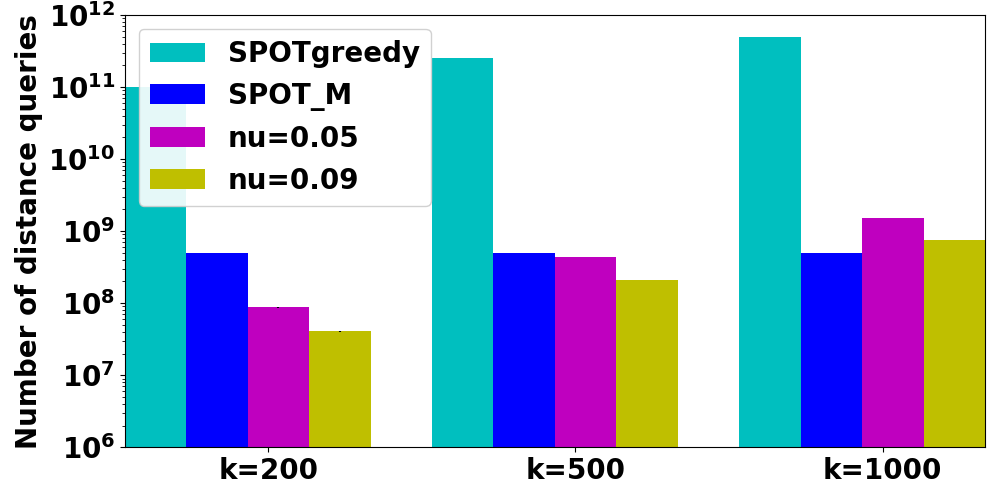}\label{subfig:cifar_distq_comp_eps02}}
    \subfigure{\includegraphics[height=3cm,width=0.32\linewidth]{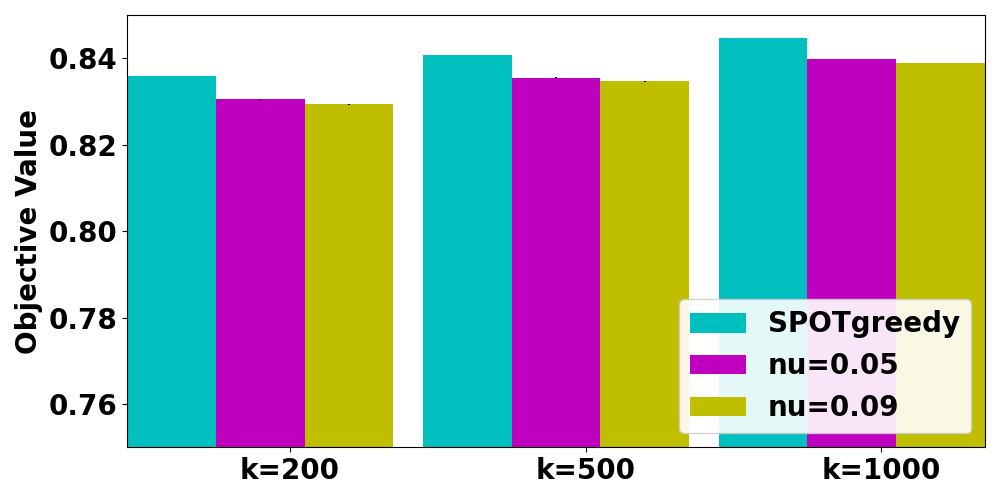}\label{subfig:cifar_obj_comp_sk_10_eps02}}
    \subfigure{\includegraphics[height=3cm,width=0.32\linewidth]{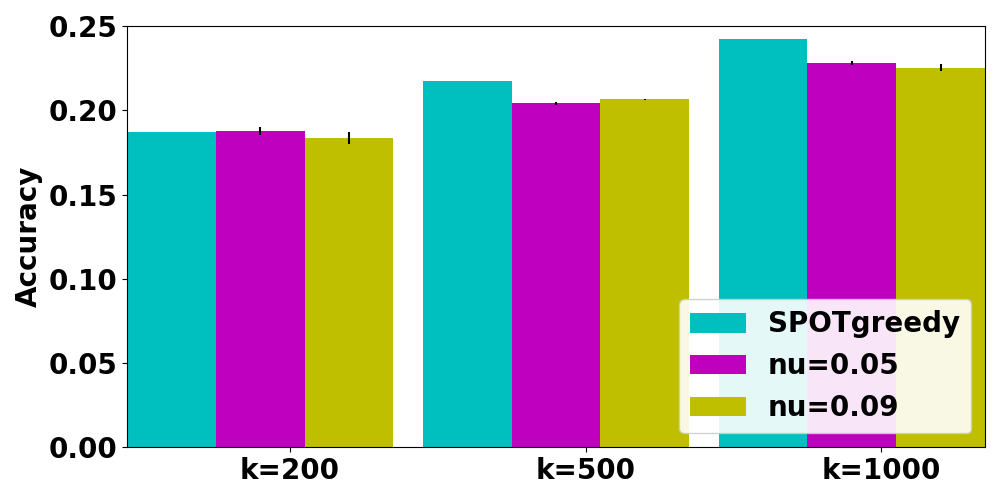}\label{subfig:cifar_acc_comp_sk_10_eps02}}
    \caption{Comparison of SPOTgreedy and \ouralgo{} (with $\epsilon=0.2$) on $\mathcal{T}$ on the CIFAR-10 dataset in terms of (a) number of similarity comparisons; (b) objective value; and (c) generalization performance. SPOT\textunderscore M in Figure~\ref{fig:cifar_comp_distq_obj_acc_eps02}(a) is an implementation of SPOTgreedy with all the $|\mathcal{S} \times \mathcal{T}|$ pairwise-similarity values memorized. Our proposed {\ouralgo} with different $\nu$ values obtain good accuracy at much lower number of similarity comparisons for SPOTgreedy. Compared to SPOT\textunderscore M, which memorizes all the $|\mathcal{S} \times \mathcal{T}|$ pair-wise similarity values, we see the benefit of {\ouralgo} (that never memorizes the similarity values).}\label{fig:cifar_comp_distq_obj_acc_eps02}
\end{figure}